\documentclass[11 pt]{article}
\usepackage{style}
\title{\LARGE\bfseries Approximate Global Convergence of Independent Learning in Multi-Agent Systems}
\author{Ruiyang Jin\textsuperscript{$*,1$}, Zaiwei Chen\textsuperscript{$\dagger,1$}, Yiheng Lin\textsuperscript{$\dagger,2$}, Jie Song\textsuperscript{$*,2$}, and Adam Wierman\textsuperscript{$\dagger,3$}\\
{\small
\textsuperscript{$*$}\textit{Peking University,} \href{mailto:jry@pku.edu.edu}{\textit{\textsuperscript{$1$}jry@pku.edu.edu}}, 
\href{mailto:jie.song@pku.edu.cn}{\textit{\textsuperscript{$2$}jie.song@pku.edu.cn}}}\\
{\small
\textsuperscript{$\dagger$}\textit{California Institute of Technology,} 
\href{mailto:zchen458@caltech.edu}{\textit{\textsuperscript{$1$}zchen458@caltech.edu}},
\href{mailto:yihengl@caltech.edu}{\textit{\textsuperscript{$2$}yihengl@caltech.edu}},
\href{mailto:adamw@caltech.edu}{\textit{\textsuperscript{$3$}adamw@caltech.edu}}
}
}
\date{\vspace{-0.4 in}}
\begin{document}
\maketitle

\begin{abstract}
Independent learning (IL), despite being a popular approach in practice to achieve scalability in large-scale multi-agent systems, usually lacks global convergence guarantees. In this paper, we study two representative algorithms, independent $Q$-learning and independent natural actor-critic, within value-based and policy-based frameworks, and provide the first finite-sample analysis for approximate global convergence. The results imply a sample complexity of $\tilde{\mathcal{O}}(\epsilon^{-2})$ up to an error term that captures the dependence among agents and characterizes the fundamental limit of IL in achieving global convergence. To establish the result, we develop a novel approach for analyzing IL by constructing a separable Markov decision process (MDP) for convergence analysis and then bounding the gap due to model difference between the separable MDP and the original one. Moreover, we conduct numerical experiments using a synthetic MDP and an electric vehicle charging example to verify our theoretical findings and to demonstrate the practical applicability of IL.
\end{abstract}

\section{Introduction}\label{sec:intro}

Reinforcement learning (RL) \cite{sutton2018reinforcement} has become a popular learning framework to solve sequential decision-making problems in the real world and has achieved great success in applications across different domains, such as Atari games \cite{mnih2015human}, the game of Go \cite{silver2017mastering}, robotics \cite{gu2017deep}, nuclear fusion control \cite{degrave2022magnetic}, etc. While the early literature on RL predominantly focused on the single-agent setting, motivated by the recent development in applications involving multi-agent interactions, e.g., in areas such as multiplayer games \cite{yang2020overview} and active voltage control \cite{wang2021multi}, multi-agent reinforcement learning (MARL) has received increasing attention.  

Compared to single-agent RL, the interplay of agents in MARL presents significant challenges for developing provably efficient learning algorithms. One of those challenges is scalability. Specifically, even when all agents share an identical interest, any MARL algorithm that uses centralized training becomes computationally intractable as the number of agents increases. To address this scalability issue, various decentralized MARL algorithms were developed, such as centralized training with decentralized execution (CTDE) \cite{lowe2017multi}, localized MARL in networked systems \cite{zhang2018fully}, and independent learning (IL) \cite{lanctot2017unified}. Among them, IL has the minimum requirements on coordination and information sharing, as each agent makes decisions purely based on its local observations, without collecting any information from other agents. Therefore, increasing the number of agents does not introduce additional communication or computational complexity in IL.

Although IL is simple, intuitive, and easy to implement, it is unclear whether agents can jointly achieve any performance guarantees through IL. Note that, in IL, each individual agent overlooks the multi-agent interactions by treating all other agents as parts of the environment, which is non-stationary due to the actions taken by other agents. Therefore, except in cases where the underlying model has a special structure, such as in zero-sum stochastic games \cite{chen2024zerosum,cai2023uncoupled} or Markov potential games \cite{leonardos2021global}, IL, in general, is not guaranteed to converge \cite{zhang2021multi}. This leads to the following fundamental question:
\begin{center}
\emph{Is it possible to achieve (approximate) global convergence for IL in multi-agent systems?}  
\end{center}

In this paper, we provide a positive answer to this question in the cooperative setting by establishing the first finite-sample bounds for the approximate global convergence of IL. In addition, we introduce the \emph{dependence level} as a novel notion to capture the fundamental limit of IL. Detailed contributions of this work are presented below.
\begin{itemize}
    \item \textbf{Approximate Global Convergence for Independent Learning.}
We consider two popular and representative forms of IL, one of which is a value-based algorithm called independent $Q$-learning (IQL), and the other is a policy-based algorithm within the actor-critic framework called independent natural actor-critic (INAC). We establish the first last-iterate finite-sample bounds measured by the global optimality gap. A key feature is that the bound asymptotically converges within a fixed error term, which is proportional to the dependence level, denoted by $\mathcal{E}$, of the Markov decision process (MDP) model. Here, the dependence level $\mathcal{E}$ characterizes how close the MARL model is to a model where each agent's transitions are independent of each other. Our finite-sample bound implies that, up to the aforementioned asymptotic error term, the sample complexity for the remaining terms to achieve an $\epsilon$-optimality is $\tilde{\mathcal{O}}(\epsilon^{-2})$, which is not improvable \cite{gheshlaghi2013minimax}. 
\item \textbf{A Novel Approach for Analyzing Independent Learning.}
The main challenge in analyzing IL is that each agent overlooks the multi-agent interactions by treating other agents as parts of the environment. To overcome this challenge, we develop a novel approach (depicted in Figure \ref{fig:analysis_idea}) that involves (1) constructing a separable MDP consisting of local transition kernels to approximate the original MDP, (2) analyzing IL as if they were implemented on the separable MDP, and finally, (3) bounding the error due to model difference between the separable MDP and the original one to establish the approximate global convergence.

The proof technique discussed above can be used beyond the analysis of IL. Specifically, in a general stochastic iterative algorithm, as long as the random process that drives the iterative algorithm can be approximated by a Markov chain (even though the original random process is not Markovian), our proof technique can be potentially applied to get finite-sample guarantees. See Appendix \ref{subsec:beyond_IL} for more details.
\begin{figure}[h]
    \centering
    \includegraphics[width=4in]{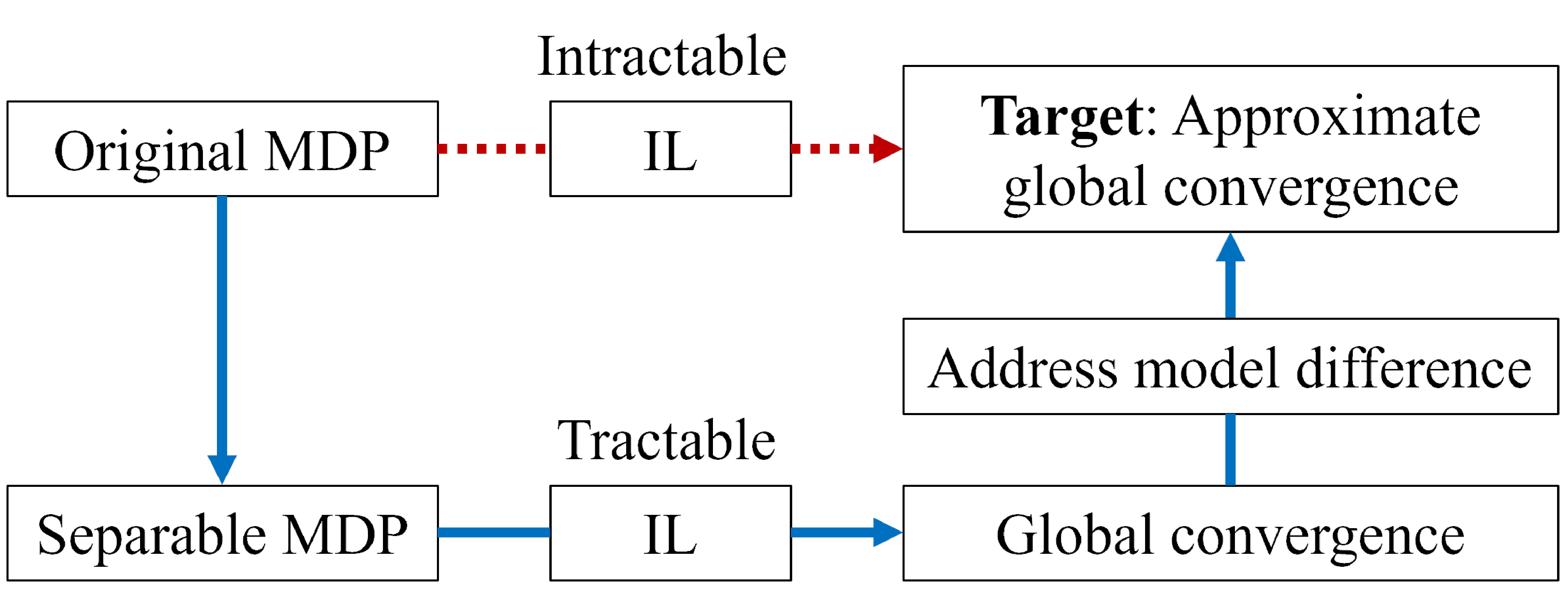}
    \caption{Our Roadmap for Analyzing IL.}
    \label{fig:analysis_idea}
\end{figure}
\item \textbf{Validation of Approximate Global Convergence of IL.} We first conduct numerical simulations on a synthetic MDP to verify the global convergence of IL, and to demonstrate that the asymptotic optimality gap of IL is determined by the dependence level. Then, we apply IL to an electric vehicle (EV) charging problem to showcase the practical applicability of our algorithms, where we employ a neural network as a form of function approximation. Notably, while our theoretical results are established for the tabular setting, our algorithms empirically achieve approximate global convergence under function approximation. 
\end{itemize}

\subsection{Related Literature}\label{sec:related_literature}

$Q$-learning and natural actor-critic are two popular and representative algorithms for value-based and policy-based methods, respectively. In the following, we discuss the related work in $Q$-learning, natural actor-critic, IL, and networked MARL.

\paragraph{$Q$-Learning.} $Q$-learning was first proposed in \cite{watkins1992q} as a data-driven variant of the $Q$-value iteration for solving the Bellman optimality equation. Due to its popularity, the convergence behavior of $Q$-learning has been extensively studied. Specifically, the asymptotic convergence of $Q$-learning has been studied in \cite{tsitsiklis1994asynchronous,jaakkola1994convergence} and finite-sample analysis in \cite{even2003learning,qu2020finite,li2023q,chen2021lyapunov}, among many others. Practically, $Q$-learning has been successfully employed in many applications in the form of the deep $Q$-network (DQN), such as games \cite{mnih2015human} and control problems \cite{kiumarsi2014reinforcement}. Although $Q$-learning has also been applied to multi-agent systems \cite{tampuu2017multiagent}, to our knowledge, the theoretical results in this domain are quite limited.

\paragraph{Natural Actor-Critic.} In RL, a popular approach for finding an optimal policy is to implement gradient-based methods directly in the policy space \cite{sutton2018reinforcement,konda2000actor}, a typical example of which is natural actor-critic \cite{bhatnagar2009natural}, where natural policy gradient \cite{kakade2001natural,xiao2022convergence,yuan2022linear} is used for the policy improvement and TD-learning \cite{sutton1988learning,bhandari2018finite,srikant2019finite} is used for the policy evaluation. Finite-sample analysis of natural actor-critic (with linear function approximation) was performed in \cite{agarwal2021theory,lan2023policy,khodadadian2022finite,chen2022ADP} and the references therein, where the state-of-the-art sample complexity is $\Tilde{\mathcal{O}}(\epsilon^{-2})$. However, they all focused on the single-agent setting. In MARL where each agent independently performs natural actor-critic, existing results do not directly extend due to multi-agent interactions.

\paragraph{Independent Learning in MARL.} IL has been widely applied in various domains, such as power systems \cite{ding2022target,jin2022deep,foruzan2018reinforcement}, communication networks \cite{liang2019spectrum, nguyen2019distributed,zhong2019deep}, etc. While IL is empirically popular, it has been demonstrated in \cite{tan1993multi} that IL may fail for tasks where coordination among agents is needed. Therefore, it is vital to understand when it is possible for IL to achieve global convergence \cite{busoniu2008comprehensive,zhang2021multi}. Although IL has been theoretically justified to some extent in certain multi-agent scenarios, such as zero-sum stochastic games \cite{chen2024zerosum,daskalakis2020independent} and Markov potential games \cite{ding2022independent,leonardos2021global}, the results are still lacking for the general setting. In this work, we establish finite-sample guarantees of IL in the cooperative setting and provide a characterization of the optimality gap based on the dependence level.

\paragraph{Networked MARL.} In networked MARL, each agent engages in information exchange only with its neighbors through an interaction network. Such localized interaction structure is commonly found in applications that involve social networks, computer networks, traffic networks, etc. \cite{lin2021multi,zhou2023b,qu2022scalable}, and it enables decentralized decision-making based on local observations and shared information from neighboring agents \cite{chu2020multi}. Theoretically, the authors of \cite{zhang2018fully} incorporate a consensus algorithm in the design of their networked MARL algorithm, and provide convergence analysis under linear function approximation. Many algorithms and analyses are proposed for the extended settings, such as continuous spaces \cite{zhang2018networked}, and stochastic networked MARL \cite{zhang2021finite}. A more detailed review of networked MARL can be found in \cite{zhang2021multi}. The works most relevant to IL within networked MARL are the scalable algorithms designed using the exponential decay property \cite{qu2022scalable,lin2021multi,qu2020scalable,zhang2023global,zhou2023b}, where the value functions of each agent have exponentially decaying correlations with agents far away from it. Convergence results with respect to the range of neighbors from which each agent collects information have been established. However, there are two key differences between our setting and this line of work: First, we do not impose the assumption of local interaction structure which requires the next local state of an agent to be only affected by the current states of its direct neighbors; Second, our IL approach only requires local information to implement, while the scalable algorithms in Networked MARL require communications between agents that are within a certain range.

\section{Problem Formulation}\label{sec:model}

Consider an MARL problem with $n$ agents. Let $\mathcal{S}$ and $\mathcal{A}$ be the global state space and the global action space, respectively. Given $i\in \{1,2,\cdots, n\}:=[n]$, let $\mathcal{S}^i$ (respectively, $\mathcal{A}^i$) be the state space (respectively, the action space) of agent $i$. In this work, we assume that $|\mathcal{S}||\mathcal{A}|<\infty$. Let $\mathcal{P}=\{P_a\in \mathbb{R}^{|\mathcal{S}|\times |\mathcal{S}|}\mid a\in\mathcal{A}\}$ be the set of transition probability matrices indexed by actions, i.e., $P_a(s, s')$ denotes the probability of transiting to the global state $s'$ after taking the joint action $a$ at the global state $s$. For each $i\in [n]$, let $\mathcal{R}^i:\mathcal{S}^i\times \mathcal{A}^i\mapsto [0,1]$ be the reward function of agent $i$. Note that restricting the image space of the reward function to $[0,1]$ is without loss of generality because we are working with a finite MARL problem. For any $s=(s^1,s^2,\cdots,s^n)\in\mathcal{S}$ and $a=(a^1,a^2,\cdots,a^n)\in\mathcal{A}$, the total one-stage reward is given by $\mathcal{R}(s,a)=\sum_{i\in[n]}\mathcal{R}^i(s^i,a^i)$. Let $\gamma\in (0,1)$ be the discount factor, which captures the weight we place on future rewards. We denote the MDP for the MARL problem as $\mathcal{M}=(\mathcal{S},\mathcal{A},\mathcal{P},\mathcal{R},\gamma)$.

Given a joint policy $\pi:\mathcal{S}\mapsto\Delta(\mathcal{A})$, where $\Delta(\mathcal{A})$ stands for the probability simplex on $\mathcal{A}$, its associated $Q$-function $Q_\pi\in\mathbb{R}^{|\mathcal{S}||\mathcal{A}|}$ is defined as
\begin{align*}
    Q_\pi(s,a)=\mathbb{E}_{\pi} \left[\sum_{k=0}^\infty \gamma^k\mathcal{R}(S_k,A_k)\;\middle|\;S_0=s,A_0=a\right], \; \forall\, (s,a)\in\mathcal{S}\times \mathcal{A},
\end{align*}
where we use $\mathbb{E}_\pi[\,\cdot\,]$ to indicate that the actions are chosen according to the joint policy $\pi$. We further define the value function $V_\pi \in \mathbb{R}^{|\mathcal{S}|}$ as $V_\pi (s)=\mathbb{E}_{a\sim \pi(\cdot \mid s)} [Q_\pi (s,a)]$ for all $s\in \mathcal{S}$. Given an initial distribution $\mu$ on the states, denote $V_\pi^\mu:= \mathbb{E}_{s\sim \mu} [V_\pi (s)]$ as the expected value of a policy $\pi$. The goal is to find a globally optimal joint policy $\pi$ such that $Q_\pi$ (or equivalently, $V_\pi$, $V_\pi^\mu$) is uniformly maximized for all $(s,a)$.

For an individual agent $i\in [n]$, we also define $Q_\pi^i\in\mathbb{R}^{|\mathcal{S}||\mathcal{A}|}$ as the $Q$-function of agent $i$ given policy $\pi$:
\begin{align*}
    Q_\pi^i (s,a)= \mathbb{E}_{\pi} \left[\sum_{k=0}^\infty \gamma^k\mathcal{R}^i(S^i_k,A^i_k)\;\middle|\;S_0=s,A_0=a\right], \; \forall \,(s,a)\in\mathcal{S}\times \mathcal{A}.
\end{align*}
Note that we have $Q_\pi(s,a)=\sum_iQ_\pi^i(s,a)$ for all $(s,a)$.

\subsection{The Dependence Level}

We now introduce the concept of \textit{dependence level} for the multi-agent MDP model to characterize the globally coupled state transitions.  This is a new notion for analyzing IL that we introduce in this paper. We first illustrate the concept of \textit{separable MDPs}, which is crucial for defining the dependence level of the multi-agent MDP. For any $i\in [n]$, let $\mathcal{P}^i$ be the set of $|\mathcal{S}^i|$ by $|\mathcal{S}^i|$ right stochastic matrices, and let
\begin{align*}
    \hat{\mathcal{Z}}=\left\{
    \{\hat{P}_a\}_{a\in \mathcal{A}}\, \middle | \,\exists\,\hat{P}_{a^i}\in\mathcal{P}^i\text{ s.t. }\hat{P}_a(s_1,s_2)=\prod_{i=1}^n\hat{P}_{a^i}(s_1^i,s_2^i),\forall\, s_1, s_2\in\mathcal{S}
    \right\}.
\end{align*}
We call the MDP with transition matrices $\{\hat{P}_a\}_{a\in\mathcal{A}}\in \hat{\mathcal{Z}}$ a separable MDP. Intuitively, while a separable MDP is defined on the joint state-action space, it can be decomposed into $n$ small MDPs that involve independently on their associated local state-action spaces.  We next define the dependence level of a multi-agent MDP model. 

\begin{definition}\label{def:dependence_level}
    \textit{An $n$-agent MDP with transition probability matrices $\{P_a\}_{a\in\mathcal{A}}$ is said to be $\mathcal{E}$-dependent if and only if}
    \begin{align}\label{trans_approx}
        \min_{\{\hat{P}_a\}_{a\in\mathcal{A}}\in\hat{\mathcal{Z}}}\max_{s,a}\left\|P_a(s,\cdot)-\hat{P}_a(s,\cdot)\right\|_{\text{TV}}=\mathcal{E},
    \end{align}
    \textit{where $\mathcal{E}$ is called the dependence level of the MDP, and $\| \cdot \|_{\text{TV}}$ is the total variation distance.}
\end{definition} 
\begin{remark}
    Since the quantity $\max_{s,a}\|P_a(s,\cdot)-\hat{P}_a(s,\cdot)\|_{\text{TV}}$ as a function of $\{\hat{P}_a\}_{a\in\mathcal{A}}$ is a continuous function and $\hat{\mathcal{Z}}$ is a compact set, 
    the $\min(\cdot)$ in Eq. (\ref{trans_approx}) is well-defined due to the extreme value theorem \cite{rudin1976principles}. 
\end{remark}
To understand Definition \ref{def:dependence_level}, suppose that the original MDP $\mathcal{M}$ is separable. Then the dependence level $\mathcal{E}=0$. In this case, we would expect IL to achieve global convergence because the problem is essentially to find the optimal policies of $n$ decoupled MDPs. More generally, the dependence level $\mathcal{E}$ can be positive, in which case the original MDP is not separable, and the exact global convergence may not be achievable. Therefore, the dependence level $\mathcal{E}$ serves as a measure of how close the original MDP is to the space of separable MDPs and captures the fundamental limit of IL for global convergence. Let 
\begin{align*}
    \hat{\mathcal{P}}= {\arg\min}_{\{\hat{P}_a\}_{a\in\mathcal{A}}\in\hat{\mathcal{Z}}}\max_{s,a}\big\|P_a(s,\cdot)-\hat{P}_a(s,\cdot)\big\|_{\text{TV}}
\end{align*}
and denote $\hat{\mathcal{M}}=(\mathcal{S},\mathcal{A},\hat{\mathcal{P}},\mathcal{R},\gamma)$, which consists of $n$ decoupled MDPs. We use the hat notation throughout the paper to denote the counterpart quantities for the separable MDP $\hat{\mathcal{M}}$. For example, we use $\hat{Q}_{\pi}\in \mathbb{R}^{|\mathcal{S}|\times|\mathcal{A}|}$ as the $Q$-function of a policy $\pi$ under $\hat{\mathcal{M}}$.

\begin{remark}
    In this work, although we assumed that the global reward is the summation of each agent's individual reward for simplicity of presentation, this assumption can be relaxed. Specifically, for each $i\in [n]$, let $\Lambda^i$ be the set of real-valued functions with domain $\mathcal{S}^i\times \mathcal{A}^i$. Then, by defining the reward dependence level $\mathcal{E}_r$ as $\mathcal{E}_r=\min_{\hat{\mathcal{R}}^i\in\Lambda^i,\forall\,i\in [n]}\max_{s, a}|\mathcal{R}(s, a)-\sum_{i\in[n]}\hat{\mathcal{R}}^i(s^i,a^i)|$, we can replace the original reward function $\mathcal{R}(s,a)$ by $\sum_{i\in[n]}\hat{\mathcal{R}}^i(s^i,a^i)$ in our analysis. Note that in this case, the reward-dependence level $\mathcal{E}_r$ will also appear in the convergence bounds.
\end{remark}

\subsection{An Illustrative Example}\label{ex:artificial}
In this subsection, we provide an illustrative example to further discuss the concept of dependence level. 

In multi-agent systems, local information sharing and decision-making coordination can significantly improve the performance of IL \cite{tan1993multi}. However, it is challenging to select the pertinent information to communicate and the agents with whom to collaborate. The following example illustrates how different options of grouping and collaboration can lead to different dependence levels, which significantly affects the performance of IL.

\begin{figure}[h]
    \centering
    \includegraphics[width=5 in]{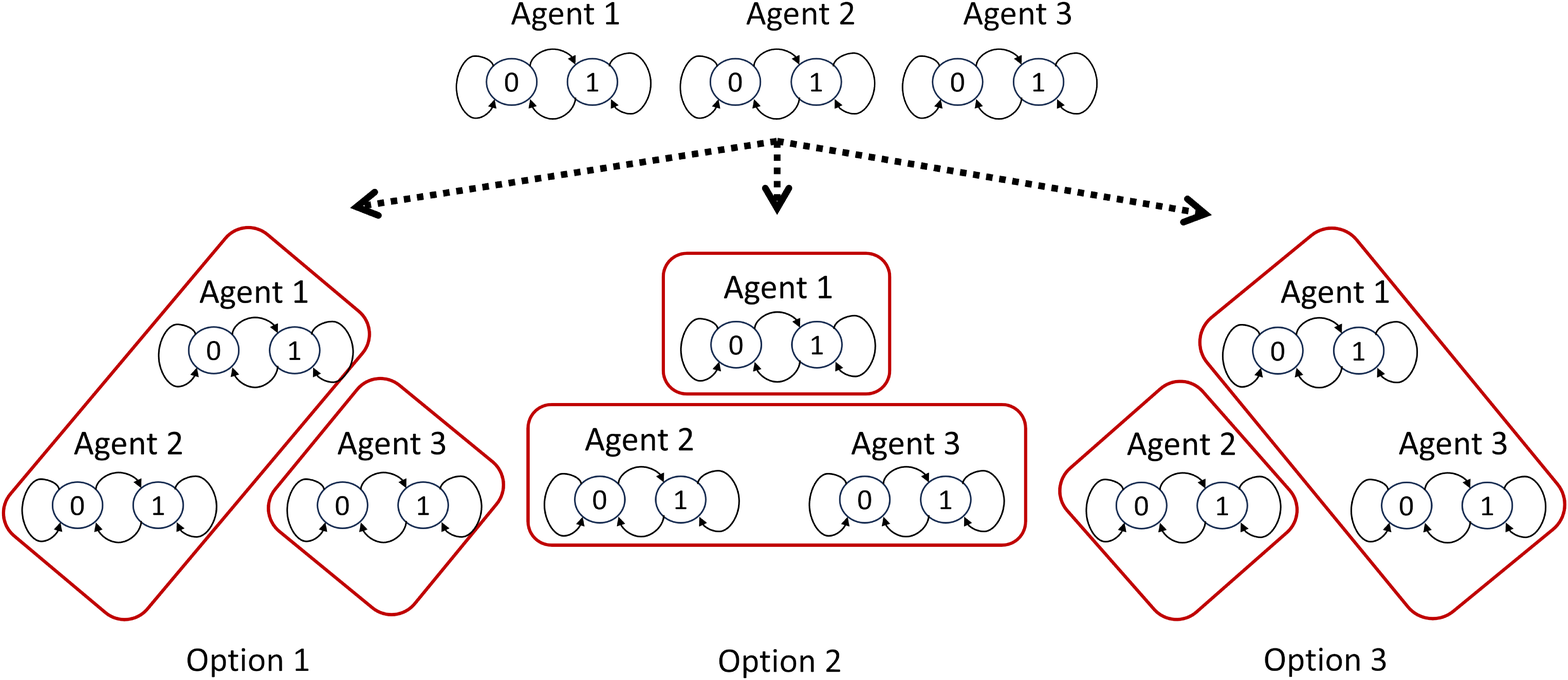}
    \caption{Illustration of the example.}
    \label{fig:example1}
\end{figure}

\begin{example}
    Consider an MDP consisting of $3$ agents. The state space of each agent is $\{0,1\}$, and the action space of each agent is also $\{0,1\}$. There are different dependencies among the states of the three agents with detailed transition probabilities shown in Appendix \ref{append:detailed_examples1}. Suppose that we are allowed to group $2$ of the $3$ agents as one agent, in which case the $2$ agents can share information and coordinate with each other. Then, we have $3$ different grouping options shown in Figure \ref{fig:example1}. After calculating the dependence level $\mathcal{E}$ for each grouping option (cf. Appendix \ref{append:detailed_examples1}), we have: Option 1 leads to $\mathcal{E}=0.5$; Option 2 leads to $\mathcal{E}=0.75$, and Option 3 leads to $\mathcal{E}=0.875$. 
\end{example}
\section{Main Results}\label{sec:main_results}
This section presents our main results. We first present the IQL and the INAC as representative IL algorithms in Sections \ref{subsec:IQL_algorithm} and \ref{susbec:INAC_algorithm}, respectively. Then, we present the finite-sample approximate global convergence guarantees for both algorithms in Section \ref{subsec:main_results}. The proof sketch of our main theorems is provided in Section \ref{subsec:proof_idea}, with detailed proof deferred to the appendix.

\subsection{Independent $Q$-Learning}\label{subsec:IQL_algorithm}

In IQL, each agent treats all other agents as parts of the environment and implements $Q$-learning on its local state-action space, as presented in Algorithm \ref{algorithm:IQL}.

\begin{algorithm}[htbp]
	\caption{Independent $Q$-Learning (Agent $i$)}
	\label{algorithm:IQL} 
	\begin{algorithmic}[1]
		\STATE \textbf{Input:} An integer $K$, a behavior policy $\pi_b^i$, and an initialization $Q^i_0=0$.
		\FOR{$k=0,1,2,\cdots,K-1$}
		\STATE Implement $A_k^i\sim \pi_b^i(\cdot\mid S_k^i)$ (simultaneously with all other agents), and observes $S^i_{k+1}$.
        \STATE Update $Q^i_k$ according to 
        \begin{align*}
            Q_{k+1}^i(S_k^i,A_k^i)=Q_k^i(S_k^i,A_k^i)+\alpha_k\left(\mathcal{R}^i(S_k^i,A_k^i) +\gamma \max_{\Bar{a}^i} Q_k^i(S_{k+1}^i,\Bar{a}^i)-Q_k^i(S_k^i,A_k^i)\right).
        \end{align*}
		\ENDFOR
        \STATE \textbf{Output:} $Q_K^i$
	\end{algorithmic}
\end{algorithm}

\paragraph{Algorithm Details.} In IQL, each agent uses a local behavior policy $\pi_b^i: \mathcal{S}^i \mapsto \Delta (\mathcal{A}^i)$ to interact with the environment to collect samples (cf. Algorithm \ref{algorithm:IQL} Line $3$) and updates its local $Q$-function estimate $Q_k^i\in \mathbb{R}^{|\mathcal{S}^i||\mathcal{A}^i|}$ according to Algorithm \ref{algorithm:IQL} Line $4$,
where $\alpha_k$ is the learning rate. Note that this is an asynchronous update because only one component (i.e., the $(S_k^i,A_k^i)$-th component) of the vector-valued local $Q$-function is updated at each step. The update of $Q$-learning can be viewed as a stochastic approximation algorithm for solving the Bellman optimality equation. See \cite{sutton2018reinforcement,bertsekas1996neuro,watkins1992q} for more details about $Q$-learning.

\begin{remark}
Unlike single-agent $Q$-learning, which is driven by a trajectory of Markovian samples,
the local sample trajectory $\{(S^i_k,A^i_k)\}_{k\geq 0}$ used for IQL, in general, does \textit{not} form a Markov chain. Therefore, the convergence of the local $Q$-function $\{Q^i_{t,k}\}_{k\geq 0}$ does not directly follow from the existing results studying single-agent $Q$-learning \cite{even2003learning,chen2021lyapunov,qu2020finite,li2023q}, which presents a challenge in the analysis. 
\end{remark}

\subsection{Independent Natural Actor-Critic}\label{susbec:INAC_algorithm}
The detailed description of the INAC algorithm is presented in Algorithm \ref{algorithm:INAC}. At a high level, INAC consists of an actor in the outer loop and a critic in the inner loop. The actor uses independent natural policy gradient (INPG) to update the policy and the critic uses independent TD-learning (ITD) to estimate the local $Q$-function, which is needed for the actor. We next elaborate on the actor and the critic in more detail.

\begin{algorithm}[htbp]
	\caption{Independent Natural Actor-Critic (Agent $i$)}
	\label{algorithm:INAC} 
	\begin{algorithmic}[1]
		\STATE \textbf{Input:} Integers $K,T$, initializations $\theta_0^i=0$ and $Q^i_{t,0}=0$ for all $t\geq 0$.
		\FOR{$t=0, 1,2,\cdots,T-1$}
		\FOR{$k=0,1,2,\cdots,K-1$}
		\STATE Implement $A_k^i\sim \pi^i_{\theta_t^i}(\cdot\mid S_k^i)$ (simultaneously with all other agents), and observes $S_{k+1}^i$.
        \STATE Update $Q_{t,k}^i$ according to \begin{align*}
            Q_{t,k+1}^i(S_k^i,A_k^i) =(1-\alpha_k)Q_{t,k}^i(S_k^i,A_k^i) +\alpha_k \left(\mathcal{R}^i(S_k^i,A_k^i)+\gamma \mathbb{E}_{a^i\sim \pi_{(t)}^i(\cdot\mid S_{k+1}^i)}\left[Q_{t,k}^i(S_{k+1}^i,a^i)\right]\right).
        \end{align*}
		\ENDFOR
        \STATE $\theta_{t+1}^i=\theta_t^i+\eta_t Q_{t,K}^i$.
		\ENDFOR
  \STATE \textbf{Output:} $\pi_{\theta_T^i}^i$
	\end{algorithmic}
\end{algorithm} 

\paragraph{Independent Natural Policy Gradient for the Actor.} Policy gradient \cite{sutton2018reinforcement} is a popular approach for solving the RL problem. The idea is to perform gradient ascent in the policy space. NPG can be viewed as a variant of the policy gradient, where the Fisher information matrix is used as a preconditioner \cite{kakade2001natural}. See \cite{kakade2001natural,agarwal2021theory,lan2023policy} for other equivalent formulations and interpretations of NPG.

To present the NPG algorithm, we consider using softmax policies with parameter $\theta\in\mathbb{R}^{|\mathcal{S}||\mathcal{A}|}$: 
\begin{align*}
    \pi_\theta(a\mid s)=\frac{\exp{(\theta_{s,a})}}{\sum_{\Bar{a}\in\mathcal{A}}\exp{(\theta_{s,\Bar{a}})}},\; \forall\, (s,a)\in\mathcal{S}\times \mathcal{A}.
\end{align*}
It was shown in \cite{agarwal2021theory} that when using softmax policies, NPG takes the following form in the parameter space (which is also called $Q$-NPG):
\begin{align}\label{OriginalNPG}
    \theta_{t+1}=\theta_t+\eta_t Q_{(t)},
\end{align}
where $\eta_t$ is the stepsize and we denote $Q_{(t)}=Q_{\pi_{\theta_t}}$. In addition, the previous update equation in the parameter space is equivalent to the following update equation in the policy space \cite{agarwal2021theory}:
\begin{align}\label{OriginalNPGpolicy}
    \pi_{(t+1)}(a\mid s)=\frac{\pi_{(t)}(a\mid s)\exp\{\eta_t Q_{(t)} (s,a)\}}{\sum_{a'\in \mathcal{A}}\pi_{(t)}(a' \mid s)\exp\{\eta_t Q_{(t)}(s, a')\}}, \; \forall\, (s,a)\in\mathcal{S}\times \mathcal{A},
\end{align}
where we denote $\pi_{(t)}=\pi_{\theta_t}$ for simplicity of presentation. 
However, carrying out the update rule in Eq. (\ref{OriginalNPG}) (or Eq. (\ref{OriginalNPGpolicy})) would require the agents to jointly estimate the global $Q$-function, which, in general, cannot be achieved with IL. 

To enable the use of IL, we propose that agent $i$ maintains its own parameter $\theta_t^i\in\mathbb{R}^{|\mathcal{S}^i||\mathcal{A}^i|}$ and updates it according to
\begin{align}\label{INPG}
    \theta_{t+1}^i=\theta_t^i+\eta_t \hat{q}_{(t)}^i,
\end{align}
where $\hat{q}_{(t)}^i\in \mathbb{R}^{|\mathcal{S}^i||\mathcal{A}^i|}$ is the local $Q$-function\footnote{Here, we use the notation $\hat{q}_{\pi}^i$ to distinguish with the $Q$-function $\hat{Q}_\pi^i$ given policy $\pi$ and agent $i$, which is defined in the global state-action space. } of agent $i$ associated with the policy $\pi_{(t)}$ under the \textit{separable MDP} model $\hat{\mathcal{M}}$.
Specifically, given a separable joint policy $\pi=(\pi^1,\pi^2,\cdots,\pi^n)$ with $\pi^i: \mathcal{S}^i\mapsto \Delta(\mathcal{A}^i)$ for any agent $i$, we define $\hat{q}_\pi^i(s^i,a^i)=\hat{\mathbb{E}}_{\pi}[\sum_{k=0}^\infty\gamma^k\mathcal{R}^i(S_k^i,A_k^i)\mid S_0^i=s^i,A_0^i=a^i]$
for all $(s^i,a^i)\in\mathcal{S}^i\times \mathcal{A}^i$ and $i\in [n]$,
where $\hat{\mathbb{E}}_\pi[\,\cdot\,]$ denotes the expectation with respect to the separable transition kernel $\hat{\mathcal{P}}$ and policy $\pi$. 
To make sense of Eq.~\eqref{INPG}, note that Eq. \eqref{INPG} is equivalent to the following update in the policy space:
\begin{align*}
    \pi^i_{(t+1)}(a^i\mid s^i)=\frac{\pi^i_{(t)}(a^i\mid s^i)\exp\{\eta_t \hat{q}_{(t)}^i (s^i,a^i)\}}{\sum_{\Bar{a}^i \in \mathcal{A}^i}\pi^i_{(t)}(\Bar{a}^i \mid s^i)\exp\{\eta_t \hat{q}_{(t)}^i(s^i, \Bar{a}^i)\}},\;\forall\,(s^i,a^i)\in\mathcal{S}^i\times \mathcal{A}^i.
\end{align*}
Suppose that the original MDP $\mathcal{M}$ itself is separable, in which case we have $\mathcal{M}=\hat{\mathcal{M}}$. Then, for any $s=(s^1,s^2,\cdots,s^n)\in\mathcal{S}$ and $a=(a^1,a^2,\cdots,a^n)\in\mathcal{A}$, it is clear that $Q_\pi(s,a)=\sum_{i\in [n]}\hat{q}_{\pi}^i(s^i,a^i)$. As a result, when each agent updates its policy parameter $\theta_t^i$ according to Eq. (\ref{INPG}), the joint policy obeys the following update rule:
\begin{align*}
    \pi_{(t+1)}(a\mid s)=\;&\prod_{i=1}^n\pi_{(t+1)}^i(a^i\mid s^i)\\
    =\;&\prod_{i=1}^n\left[\frac{\pi^i_{(t)}(a^i\mid s^i)\exp\{\eta_t \hat{q}_{(t)}^i (s^i,a^i)\}}{\sum_{\Bar{a}^i \in \mathcal{A}^i}\pi^i_{(t)}(\Bar{a}^i \mid s^i)\exp\{\eta_t^i \hat{q}_{(t)}^i(s^i, \Bar{a}^i)\}}\right]\\
    =\;&\frac{\pi_{(t)}(a\mid s)\exp\{\eta_t Q_{(t)} (s,a)\}}{\sum_{\Bar{a}\in \mathcal{A}}\pi_{(t)}(\Bar{a} \mid s)\exp\{\eta_t Q_{(t)}(s, \Bar{a})\}},\;\forall\,(s,a)\in\mathcal{S}\times \mathcal{A},
\end{align*}
which is exactly the desired $Q$-NPG update presented in Eq. (\ref{OriginalNPGpolicy}). In general, when the original $n$-agent MDP $\mathcal{M}$ is not separable, Eq. (\ref{INPG}) can be viewed as an approximation of the $Q$-NPG update in Eq. (\ref{OriginalNPGpolicy}). Explicitly characterizing such an approximation error is one of the major technical challenges in the analysis. As we shall see later, the approximation error will be captured by the dependence level $\mathcal{E}$. 

In view of Eq. \eqref{INPG}, each agent needs to estimate its local $Q$-function $\hat{q}_{(t)}^i$ to carry out the update. To achieve that, we use ITD, which is presented next.

\paragraph{Independent TD-Learning for the Critic.} Within each iteration $t\in \{0,1,\cdots,T-1\}$ of the outer loop, each agent performs policy evaluation independently according to Algorithm \ref{algorithm:INAC} Line $5$. Similarly to $Q$-learning, TD-learning can also be viewed as a stochastic approximation algorithm for solving the Bellman equation for policy evaluation. See \cite{sutton2018reinforcement,bertsekas1996neuro,sutton1988learning} for more details about TD-learning.

Finally, combining the INPG in Eq. \eqref{INPG} with the ITD for estimating $\hat{q}_{(t)}^i$ leads to the INAC presented in Algorithm \ref{algorithm:INAC}.

\subsection{Finite-Sample Analysis} \label{subsec:main_results}
To present our theoretical results, we first introduce our assumption regarding the MDP model. For any policy $\pi$, state-action pairs $(s,a), (\Bar{s},\Bar{a})\in \mathcal{S}\times \mathcal{A}$, and $k\geq 0$, let $\mathbb{P}^\pi_k(\Bar{s},\Bar{a},s,a)$ be the probability of visiting $(\Bar{s},\Bar{a})$ at time step $k$ starting at $(s,a)$ and following the policy $\pi$ thereafter, that is, $\mathbb{P}^\pi_k(\Bar{s},\Bar{a},s,a)=\mathbb{P}((S_k,A_k)=(s,a)\mid (S_0,A_0)=(\Bar{s},\Bar{a}))$, where $A_\ell \sim\pi(\cdot\mid S_\ell)$ for all $\ell\in \{1,2,\cdots,k\}$. 

\begin{assumption}\label{assum_markov_chain}
\textit{For any joint policy $\pi$, the induced Markov chain $\{(S_k,A_k)\}_{k\geq 0}$ (from the original MDP $\mathcal{M}$) is irreducible and aperiodic with a unique stationary distribution $d_\pi\in\Delta(\mathcal{S}\times \mathcal{A})$. In addition, we assume that $\sigma :=\inf_\pi \min_{s,a} d_\pi (s,a)>0$, and there exist $M_1\geq 0$ and $M_2\geq 1$ such that 
\begin{align*}
       \max_{\mathcal{N}\subseteq \mathcal{S}\times \mathcal{A}} \sup_{\pi}\max_{(\Bar{s},\Bar{a})\in\mathcal{S}\times \mathcal{A}}\left| \sum_{(s,a)\in\mathcal{N}} (d_\pi(s,a)- \mathbb{P}_k^\pi(\Bar{s},\Bar{a},s,a)) \right |\leq M_1 \exp\left(-\frac{k}{M_2}\right),\;\forall\,k\geq 0.
\end{align*}}
\end{assumption}
In RL, to successfully learn an optimal policy, it is well known that having a sufficient exploration component is necessary. Assumption \ref{assum_markov_chain} is imposed to ensure the exploration of RL agents, which states that for any joint policy $\pi$, the agents can sufficiently explore the state-action space. This type of assumption is commonly imposed in the existing literature studying RL algorithms, especially those with time-varying sampling policies. See for example \cite{lin2021multi,wu2020finite,zou2019finite,zeng2022finite,khodadadian2022finite}. 

To proceed and state our main results, we need to introduce more notation. For any $i\in [n]$ and $(\Bar{s}^i,\Bar{a}^i)\in \mathcal{S}^i\times \mathcal{A}^i$, let $d_\pi'(\Bar{s}^i,\Bar{a}^i)=\sum_{(s,a)\in \mathcal{S}\times\mathcal{A}, s^i=\Bar{s}^i,a^i=\Bar{a}^i} d_\pi(s,a)$. Denote $m=\max_{i\in [n]}|\mathcal{S}^i||\mathcal{A}^i|$ and $\sigma'=\inf_\pi\min_{i\in [n], (\Bar{s}^i,\Bar{a}^i)\in \mathcal{S}^i\times \mathcal{A}^i} d_\pi'(\Bar{s}^i,\Bar{a}^i)$, which is strictly positive under Assumption \ref{assum_markov_chain}. Next, we present the finite-sample analysis of IQL. 

\begin{theorem}\label{thm_IQL}
    Consider $\{Q_k\}_{k\geq 0}$ generated by Algorithm \ref{algorithm:IQL}. Suppose that Assumption~\ref{assum_markov_chain} is satisfied and $\alpha_k=\frac{\alpha}{k+k_0}$ with $k_0=\max (4\alpha, 2M_2\log K)$ and $\alpha\geq \frac{2}{\sigma'(1-\gamma)}$. Then, for any $\delta'\in (0,1)$, with probability at least $1-\delta'$, we have
    \begin{align} \label{eq:IQL_convergence}
            V_{\pi_*}^\mu-V_{\pi_K}^\mu  \leq  \underbrace{\frac{1}{1-\gamma}\left(\frac{2nC'_a}{\sqrt{K+k_0}}+\frac{2nC_b}{K+k_0}\right)}_{E_1:\;\text{Q-Learning Convergence Error}} +\underbrace{\frac{8n\gamma\mathcal{E}}{(1-\gamma)^3}}_{E_2:\; \text{Error due to }\mathcal{E}}, 
    \end{align}
where $\pi_*$ is an optimal policy, and $\pi_K=(\pi_k^1,\pi_k^2,\cdots,\pi_k^n)$ is the policy greedily induced by $(Q_k^1,Q_k^2,\cdots,Q_k^n)$, that is, $\pi_k^i(a^i\mid s^i)=1$ if and only if $a^i=\arg\max_{\Bar{a}^i} Q_k^i(s^i,\Bar{a}^i)$ for all $i\in [n]$, where we break the tie arbitrarily, 
 $C'_a=\frac{40\alpha}{(1-\gamma)^2}\sqrt{M_2 \log  K \left( \log  \left( \frac{4mnM_2 K}{\delta'} \right)+\log\log K \right)}$, and $C_b=  8\max\{ \frac{144M_2\alpha\log K+4M_1\sigma'(1+2M_2+4\alpha)}{(1-\gamma)^2\sigma'},\frac{2M_2\log K+k_0}{(1-\gamma)^2} \}$.
\end{theorem}
\begin{remark}
    Since IQL uses a fixed behavior policy $\pi_b$ to collect samples, Assumption \ref{assum_markov_chain} can be relaxed to the following weaker assumption: the Markov chain $\{(S_k,A_k)\}$ induced by $\pi_b$ is irreducible and aperiodic.
\end{remark}

The proof of Theorem~\ref{thm_IQL} is presented in Appendix~\ref{append:proof_of_thms}. Observe that the convergence bound presented in Eq. (\ref{eq:IQL_convergence}) consists of two terms. The term $E_1$ converges to zero at a rate of $\Tilde{\mathcal{O}}(1/\sqrt{K})$, which matches the convergence rate of $Q$-learning in the single-agent setting \cite{qu2020finite,li2023q,chen2021lyapunov}. The term 
$E_2$ is asymptotically non-vanishing. Note that $E_2$ is proportional to the dependence level $\mathcal{E}$, and captures the fundamental limit of IQL. In the special case where the original MDP $\mathcal{M}$ is separable, $E_2$ vanishes and we have the global convergence of IQL.

Based on Theorem \ref{thm_IQL}, we have the following sample complexity of IQL.

\begin{corollary} \label{corollary_IQL}
    Given $\epsilon >0$, for Algorithm~\ref{algorithm:IQL} to achieve $V^\mu_{\pi_*}-V^\mu_{\pi_k} \leq \epsilon+\frac{8n\gamma\mathcal{E}}{(1-\gamma)^3}$ with probability at least $1-\delta'$, the sample complexity\footnote{It was argued in \cite{khodadadian2021finite} that, given a finite-sample bound with asymptotically non-vanishing terms on the right-hand side (RHS), the interpretation of the finite-sample bound in terms of sample complexity can be ambiguous, as it is possible to trade-off the vanishing terms and the non-vanishing terms to obtain `better' sample complexity guarantees. In Corollary \ref{corollary_IQL}, we present the sample complexity in this way to allow a fair comparison with the existing literature, as a finite-sample bound with non-vanishing terms can frequently occur in RL when (1) function approximation is used, (2) off-policy sampling is used, and (3) IL is used as in our paper. } is $\Tilde{\mathcal{O}}\left(\epsilon^{-2}\right)$.
\end{corollary}

As we see from Corollary \ref{corollary_IQL}, up to a model difference error that is proportional to the dependence level, Algorithm \ref{algorithm:IQL} achieves a sample complexity of $\Tilde{\mathcal{O}}(\epsilon^{-2})$ to find a optimal policy. This sample complexity is known to be optimal due to the existing lower bounds for solving RL problems \cite{gheshlaghi2013minimax}.

Next, we present the finite-sample analysis for INAC.

\begin{theorem}\label{thm_INAC}
    Consider $\{\pi_{(t)}\}_{t\geq 0}$ generated by Algorithm \ref{algorithm:INAC}. Suppose that (1) Assumption \ref{assum_markov_chain} is satisfied, (2) $\alpha_k=\alpha/(k+k_0)$ with $\alpha\geq\frac{2}{\sigma'(1-\gamma)}$ and $k_0=\max(4\alpha,2M_2\log K)$, and (3) $\eta_t$ satisfies $\eta_0=\gamma \log \vert \mathcal{A} \vert$ and $\eta_{t}\geq 2n\log \vert\mathcal{A} \vert \sum_{i=0}^{t-1}\eta_i /[(1-\gamma)\gamma^{2t-1}]$ for all $t\geq 0$. Then, for any $\delta\in (0,1)$, with probability at least $1-\delta$, we have
    \begin{align}\label{eq:INAC_convergence}
         \left\Vert Q_{\pi_*}-Q_{(T)} \right\Vert_\infty
         \leq 
         \underbrace{\frac{2}{(1-\gamma)^2} \left(\frac{nC_a}{\sqrt{K+k_0}}+\frac{nC_b}{K+k_0}\right)}_{G_1:\;\text{TD-Learning Convergence Error}} +\underbrace{\frac{4n\gamma^{T-1}}{(1-\gamma)^2} }_{G_2: \text{Actor Convergence Error}} + \underbrace{\frac{8n\gamma\mathcal{E}}{(1-\gamma)^4}}_{G_3: \text{Error due to }\mathcal{E}}, 
    \end{align}
    where $\pi_*$ is an optimal policy, $C_a$ is defined as $        C_a=\frac{40\alpha}{(1-\gamma)^2}\sqrt{M_2 \log  K \left( \log  \left( \frac{4mnTM_2 K}{\delta} \right)+\log\log K \right)},$
    and $C_b$ is defined in Theorem~\ref{thm_IQL}.
\end{theorem}
The detailed proof of Theorem \ref{thm_INAC} is deferred to Appendix~\ref{append:proof_of_thms}. Similarly to Theorem \ref{thm_IQL}, the convergence bound in Theorem \ref{thm_INAC} is composed of terms that converge to $0$ asymptotically and a term that is proportional to the dependence level $\mathcal{E}$. Specifically, the terms $G_1$ and $G_2$ in Eq. \eqref{eq:INAC_convergence} represent the convergence error in ITD for the critic and INPG for the actor. The $\mathcal{O}(1/\sqrt{K})$ convergence rate of $G_1$ and the geometric convergence of $G_2$ agree with the existing results in the literature on TD-learning \cite{bhandari2018finite,srikant2019finite} and NPG \cite{chen2023approximate,lan2023policy}. The term $G_3$ is a constant independent of the number of iterations and captures the model difference error between the original MDP and the separable one, which is proportional to the dependence level $\mathcal{E}$. 

At first glance, it may seem unnatural that our stepsize sequence $\{\eta_t\}$ for INAC is increasing (see the recursive geometric form in Theorem \ref{thm_INAC}). To illustrate this, consider the single-agent setting. In view of the equivalent form of $Q$-NPG in Eq. (\ref{OriginalNPGpolicy}), it resembles the classical policy iteration (which has geometric convergence) when $\eta_t$ approaches infinity. This intuition was theoretically justified in \cite{xiao2022convergence,chen2022ADP,khodadadian2022linear}, as well as from the perspective of mirror descent \cite{lan2023policy}. Although using increasing stepsizes is theoretically justified, in practice, excessively large stepsizes may discourage exploration. To overcome this practical issue, one can use a more exploration-encouraging variant of INAC, such as choosing actions based on $\pi^i_{(t)}$ with probability $1-\epsilon$ and choosing actions uniformly at random with probability $\epsilon$ for some $\epsilon>0$.

Next, we derive the sample complexity of Algorithm \ref{algorithm:INAC} based on Theorem \ref{thm_INAC}. 

\begin{corollary}\label{corollary_INAC}
Given $\epsilon >0$, for Algorithm~\ref{algorithm:INAC} to achieve $  \left\Vert Q_{\pi_*}-Q_{(t)} \right\Vert_\infty \leq \epsilon+\frac{8n\gamma\mathcal{E}}{(1-\gamma)^4}$ with probability at least $1-\delta$, the sample complexity is $\Tilde{\mathcal{O}}\left(\epsilon^{-2}\right)$.
\end{corollary}

Similarly to IQL, we have an $\Tilde{\mathcal{O}}(\epsilon^{-2})$ sample complexity for INAC to find a global optimal policy up to a model difference error. 

\paragraph{Comparison with Results for Networked MARL.}
In the existing literature, the work closest to ours are those analyzing networked MARL problems. The main idea in networked MARL is that, from each agent's perspective, the agents that are far away in graph distance should have a negligible impact on the agent. This is referred to as the `exponential decay property' in networked MARL \cite{qu2022scalable,lin2021multi}. Therefore, by restricting the sharing of information between agents within their $\kappa$-hop neighborhood, with a properly chosen $\kappa$, decentralized RL algorithms can achieve scalability without compromising too much on optimality. Compared to \cite{qu2022scalable,lin2021multi}, our algorithm does not require any information exchange among agents. Furthermore, to rigorously establish the exponential decay property, certain assumptions must be imposed on the underlying MDP model \cite{qu2022scalable,lin2021multi}. In this work, we do not impose structural assumptions on the underlying model except the one that guarantees exploration (cf. Assumption \ref{assum_markov_chain}), which is commonly made in the existing literature. While our results are more general and applicable to networked MARL, the usefulness depends on how small the dependence level is. The local structure in networked MDP (which limits the impact of far-away agents) may help reduce the dependence level. 

\subsection{Proof Sketch} \label{subsec:proof_idea}
Our proof follows the roadmap in Figure \ref{fig:analysis_idea}. Next, we present the proof sketch of Theorem \ref{thm_INAC}. The proof of Theorem \ref{thm_IQL} follows a similar approach. The proof consists of the following $3$ main steps. 

\paragraph{Step $1$: Convergence of ITD.} The main challenge in analyzing ITD (and also IQL) is that, as a stochastic approximation algorithm, the randomness in the algorithm comes from the stochastic process $\{(S_k^i,A_k^i)\}$, which does not necessarily form a Markov chain. Therefore, the existing results on the Markovian stochastic approximation \cite{srikant2019finite,chen2022automatica} do not apply directly here. To overcome this challenge, inspired by \cite{lin2021multi}, we model ITD (and also IQL) as a stochastic approximation algorithm with state aggregation and show that $Q_{t,K}^i$ approximates a solution to a variant of the projected Bellman equation, which is denoted as $\Tilde{Q}_t^i$.
In the end, we obtain the following convergence result with high probability: $\Vert Q_{t,K}^i - \Tilde{Q}_t^i \Vert_\infty \leq \Tilde{\mathcal{O}}(1/\sqrt{K})$.

\paragraph{Step $2$: Global Convergence on the Separable MDP.} Following our blueprint described in Figure \ref{fig:analysis_idea}, we analyze INAC as if it were implemented in the separable MDP $\hat{\mathcal{M}}$. Then, we take a further step for ITD to prove $\Vert \Tilde{Q}_t^i-\hat{q}_{(t)}^i\Vert_\infty\leq \mathcal{O}(\mathcal{E})$ using the definition of the dependence level. Altogether, we obtain the approximate convergence for ITD: $\Vert Q_{t,K} - \hat{Q}_{(t)} \Vert_\infty \leq \Tilde{\mathcal{O}}(1/\sqrt{K})+\mathcal{O}(\mathcal{E})$,
where $Q_{t,k}(s,a)=\sum_{i\in [n]}Q_{t,k}^i(s^i,a^i)$ for all $(s,a)\in\mathcal{S}\times\mathcal{A}$, and $\hat{Q}_{(t)}$ is the $Q$-function of the policy $\pi_{(t)}$ on $\hat{\mathcal{M}}$. Combining the results for ITD and INPG, we have the following result for INAC:
\begin{align}
    \Vert Q_{(T)}-\hat{Q}_{\hat{\pi}_*} \Vert_\infty \leq \Tilde{\mathcal{O}}(1/\sqrt{K})+\mathcal{O}(\gamma^T)+\mathcal{O}(\mathcal{E})\label{eq:sketch1}
\end{align}
with high probability,
where $\hat{\pi}_*$ is an optimal policy of $\hat{\mathcal{M}}$.

\paragraph{Step $3$: Bounding the Model Difference Error.} With the approximate convergence to the optimal $Q$-function of $\hat{\mathcal{M}}$, the last step is to bound the gap due to the model difference to get the approximate global convergence of the original MDP. In fact, we have $  \Vert Q_{\pi_*} -\hat{Q}_{\hat{\pi}_*} \Vert_\infty \leq \mathcal{O}(\mathcal{E})$, where we recall that $\mathcal{E}$ is the dependence level. Combining the above inequality with Eq. (\ref{eq:sketch1}) finishes the proof.

\section{Numerical Simulations}\label{sec:simulation}
Our last technical section presents numerical experiments for IL. First, we present the results of INAC and IQL applied to the synthetic MDP discussed in Section \ref{ex:artificial} to illustrate the effects of the dependence level. In Appendix \ref{subsec:ev_charging_results}, we apply IQL and INAC to an EV charging problem to demonstrate that our algorithms can be extended to the function approximation setting with approximate global convergence as well.

\begin{figure}[htbp]
    \centering
    \includegraphics[width=6in]{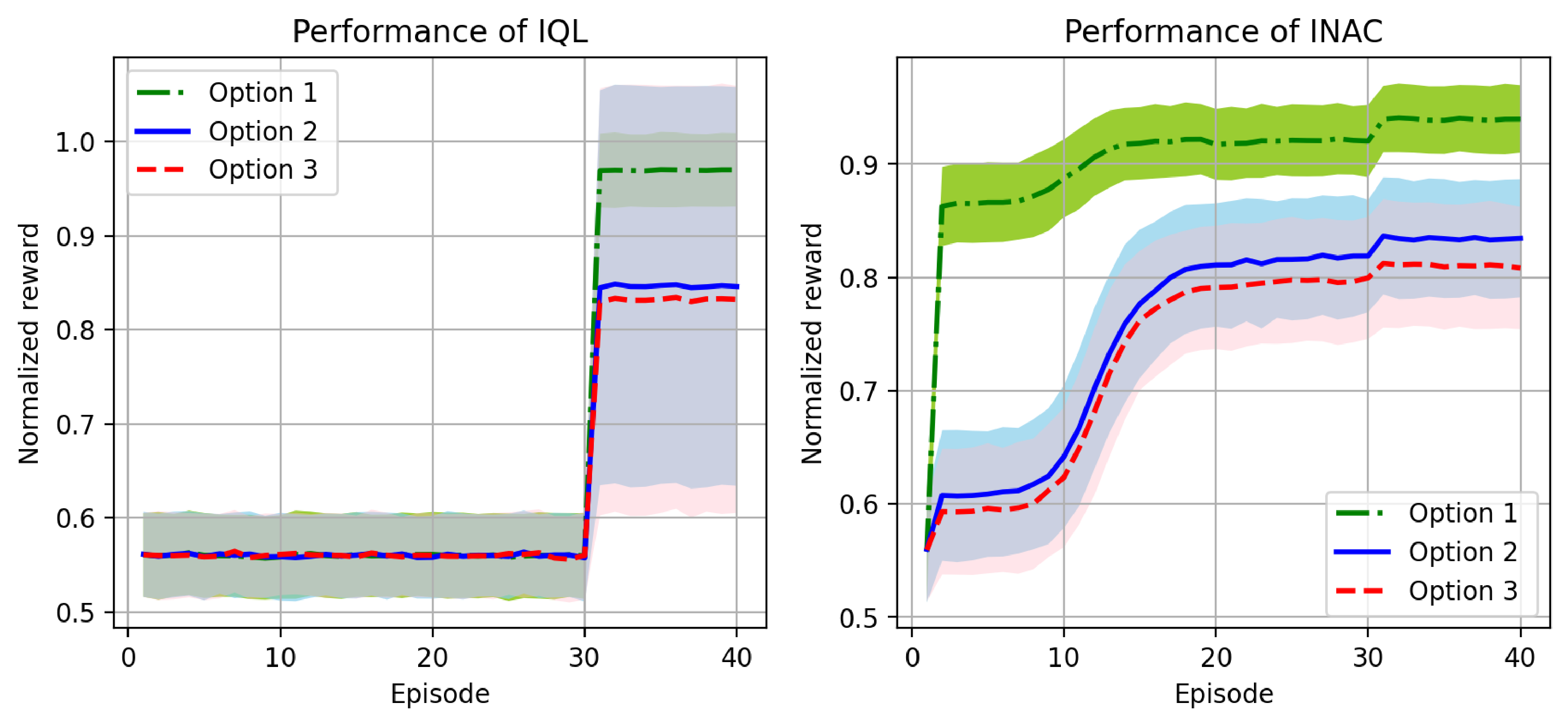}
    \caption{Performance of IQL and INAC}
    \label{fig:artificila_combined_result}
\end{figure}
We run IQL and INAC with the synthetic MDP $100$ times, respectively. The rewards are normalized by the optimal average reward and shown in Figure~\ref{fig:artificila_combined_result}, where the shaded areas denote the standard deviation. From Figure~\ref{fig:artificila_combined_result}, we see that both IQL and INAC with Option 1 can achieve better performance than Options 2 and 3, which justifies the convergence results in Theorems \ref{thm_IQL} and \ref{thm_INAC}, where the optimality gap of IL is controlled by the dependence levels. Comparing IQL with INAC, we find that IQL can achieve a smaller gap than INAC with any option, which is also consistent with our theoretical results where the asymptotically non-vanishing term (cf. $E_2$) in IQL is smaller than the corresponding term $G_3$ in INAC by a factor of $1/(1-\gamma)$.

\section{Conclusion}\label{sec:conclusion}
In this paper, we investigate the theoretical merit of IL for MARL in the cooperative setting and establish approximate global convergence for $2$ representative algorithms: IQL and INAC, both of which achieve an $\Tilde{\mathcal{O}}(\epsilon^{-2})$ sample complexity to find a global optimal policy up to an error term that is proportional to the dependence level. Methodologically, we propose a new method for analyzing IL by constructing a separable MDP, where each agent has an independent local state transition model. The model difference between the original MDP and the separable one is captured by the dependence level. Our numerical experiments justify the theoretical findings. 

There are many interesting directions for future work. First, it is worth investigating if our proof technique can be used to study other problems beyond IL, such as general non-Markovian stochastic iterative algorithms. Second, our analysis suggests that carefully adding coordination and information sharing may reduce the dependence level of the model while maintaining the scalability of the algorithm. However, theoretically characterizing the trade-off between scalability and optimality is still an open question. 

\bibliographystyle{apalike}
\bibliography{reference}

\newpage

\begin{center}
    {\LARGE\bfseries Appendices}
\end{center}

\appendix

\section{Proofs of Theorem \ref{thm_IQL} and Theorem \ref{thm_INAC}}
\label{append:proof_of_thms}
It is intractable to directly analyze IL using existing off-the-shelf RL theory since almost all of them require the samples to be Markovian or independent and identically distributed (i.i.d.), which is not the case for local state-actions that are essentially partial observations from the overall MDP model. Fortunately, note that there are some specific connections between IL and stochastic approximation with state aggregation, where an agent maintains an estimated vector with a smaller size than that of the global state space. Therefore, we will take a first step by showing that the update of IQL and ITD can be equivalently formulated as a stochastic approximation algorithm with state aggregation and derive that the corresponding algorithm converges to a fixed point. We further bound the error between the fixed point and the target $Q$-function on the separable MDP. Finally, we bound the optimality gap between the original and separable MDPs to finish the analysis.

\subsection{Connections with Stochastic Approximation with State Aggregation}
We follow the proof idea described in Section \ref{subsec:proof_idea} by first deriving the convergence of IQL and the inner loop of INAC, which is ITD. To make this paper self-contained, we present the results of stochastic approximation with state aggregation \cite{lin2021multi} in Appendix \ref{SA_with_state_aggregation}. Here, we show that IQL and ITD are special cases of such a stochastic approximation algorithm.

For stochastic approximation with state aggregation (see Appendix \ref{SA_with_state_aggregation}), the agent only maintains a vector of function values the entries of which are abstract states with a much smaller size than that of the original state space. A surjection is defined as a mapping from the original state space to the abstract state space to decide which entry should be updated at each step.

In IQL (cf. Algorithm \ref{algorithm:IQL}) and ITD (cf. Algorithm \ref{algorithm:INAC_inner}), each agent cannot estimate the $Q$-function for entries of global state-action pairs only with visibility of local states and actions. Agent $i$ only maintains a local $Q$-function $Q_k^i\in \mathbb{R}^{|\mathcal{S}^i| |\mathcal{A}^i|}$. Note that the notation of local $Q$-function is $Q_{t,k}^i$ for ITD in INAC, but the subscription $t$ is omitted here for simplicity (due to the nested-loop structure of the algorithm). Furthermore, agent $i$ uses $Q_k^i(\Bar{s}^i,\Bar{a}^i)$ to approximate the 
$Q$-function of global state-action $(s,a)$ when $(\Bar{s}^i,\Bar{a}^i)\in \mathcal{S}^i\times\mathcal{A}^i,\; (s,a)\in \mathcal{S}\times\mathcal{A}$ and $(s^i,a^i)=(\Bar{s}^i,\Bar{a}^i)$. Therefore, IQL and ITD can be both seen as a special form of stochastic approximation with state aggregation in a very natural way. Specifically, define the surjection $h_1: \mathcal{S}\mapsto \mathcal{S}^i$ and $h_2: \mathcal{A}\mapsto \mathcal{A}^i$ as
\[h_1(s)=s^i,\qquad h_2(a)=a^i, \]
for all $s\in\mathcal{S}$ and $a\in\mathcal{A}$. That means when the global state-action is $(s,a)$, agent $i$ uses the surjections $h_1(s)$ and $h_2(a)$ to decide which local entry should be updated. Formally, define the matrix $\Phi^i\in \mathbb{R}^{|\mathcal{S}||\mathcal{A}|\times|\mathcal{S}^i||\mathcal{A}^i|}$ as
\[
\Phi^i(s,a,\Bar{s}^i,\Bar{a}^i)=
\begin{cases}
    1, & \text{if} \ h_1(s)=\Bar{s}^i, \ h_2(a)=\Bar{a}^i, \\
    0, & \text{otherwise},
\end{cases}
\;\forall\, (s,a)\in\mathcal{S}\times \mathcal{A}, (\Bar{s}^i,\Bar{a}^i)\in \mathcal{S}^i\times\mathcal{A}^i.
\]
The update of IQL and ITD for any agent $i\in[n]$ can be written as a general stochastic approximation algorithm:
\begin{align}\label{eq:IQL_ITD_SA}
     Q_{k+1}^i(h_1(S_k),h_2(A_k))
    =\; & Q_k^i(h_1(S_k),h_2(A_k)) \nonumber\\
    & +\alpha_k \left( \left[F^i(\Phi^i Q_k^i)\right](S_k,A_k)-Q_k^i\left(h_1(S_k),h_2(A_k)\right)+w_k^i\right),
\end{align}
and $Q_{k+1}^i(s^i,a^i)=Q_k^i (s^i,a^i)$ for all $(s^i,a^i)\neq (h_1(S_k),h_2(A_k))$. Here in Eq. (\ref{eq:IQL_ITD_SA}), $F^i: \mathbb{R}^{|\mathcal{S}||\mathcal{A}|}\rightarrow \mathbb{R}^{|\mathcal{S}||\mathcal{A}|}$ is an operator and can be defined differently for IQL and ITD. Note that $\{ (S_k,A_k)\}_{k\geq 0}$ is sampled using different policies, i.e., $\pi_b$ for IQL and $\pi_{(t)}$ for ITD. For simplicity, here we do not distinguish them in notation with the same assumption (Assumption \ref{assum_markov_chain}) on them.

For IQL, the operator $F^i(\cdot)$ is the Bellman optimality operator, which is defined for any $Q\in\mathbb{R}^{|\mathcal{S}||\mathcal{A}|}$ as
\begin{align}\label{eq:IQL_operator}
    [F^i(Q)](s,a)=\mathcal{R}^i(s^i,a^i)+ \gamma\mathbb{E}_{\Bar{s}\sim P_a (s,\cdot)} \left[\max_{\Bar{a}\in\mathcal{A}}Q(\Bar{s},\Bar{a})\right], \; \forall\, (s,a)\in\mathcal{S}\times\mathcal{A}.
\end{align}
Consequently, the noise sequence $w_k^i$ for IQL is defined as
\begin{align}\label{eq:IQL_noise}
    w_k^i=\mathcal{R}^i(S_k^i,A_k^i)+\gamma \max_{\Bar{a}}Q_k^i(h_1(S_{k+1}),h_2(\Bar{a}))-[F^i(\Phi^i Q_k^i)](S_k,A_k).
\end{align}
For ITD, we define the operator $F^i(\cdot)$ for any $Q\in\mathbb{R}^{|\mathcal{S}||\mathcal{A}|}$ as
\begin{align}\label{eq:ITD_operator}
    [F^i(Q)](s,a)=\mathcal{R}^i(s^i,a^i)+ \gamma\mathbb{E}_{\Bar{s}\sim P_a (s,\cdot),\Bar{a}\sim \pi(\cdot \mid \Bar{s})} \left[Q(\Bar{s},\Bar{a})\right], \; \forall \,(s,a)\in\mathcal{S}\times\mathcal{A},
\end{align}
where $\pi=\pi_{(t)}$ is the policy within the $t$-th outer loop in Algorithm \ref{algorithm:INAC}. In ITD, the noise sequence $w_k^i$ is defined as
\begin{align}\label{eq:ITD_noise}
    w_k^i=\mathcal{R}^i(S_k^i,A_k^i)+\gamma \mathbb{E}_{\Bar{a}\sim \pi(\cdot\mid S_{k+1})} [Q_k^i(h_1(S_{k+1}),h_2(\Bar{a}))]-[F^i(\Phi^i Q_k^i)](S_k,A_k).
\end{align}
After reformulating IQL and ITD as stochastic approximation with state aggregation, we will present the convergence results of IQL and ITD using \cite[Theorem 3.1]{lin2021multi}. We first introduce the following lemma that provides the required conditions for our results. Note that we only consider the $\ell_\infty$ norm here. Define $\mathcal{F}_k$ as the $\sigma$-algebra generated by $(S_0, A_0, \cdots, S_k,A_k )$.

\begin{lemma}\label{lemma:property_for_SA}
Suppose that Assumption \ref{assum_markov_chain} is satisfied. Let $F^i(\cdot)$ be defined as in Eq. \eqref{eq:IQL_operator} (or Eq. \eqref{eq:ITD_operator}), and let $w_k^i$ be defined in Eq. \eqref{eq:IQL_noise} (or Eq. \ref{eq:ITD_noise}). Then, we have the following results.
    \begin{enumerate}[(1)]
        \item The operator $F^i(\cdot)$ is a $\gamma$-contraction mapping with respect to $\|\cdot\|_{\infty}$. In addition, we have $\Vert F^i(Q) \Vert_\infty \leq \gamma \Vert Q\Vert_\infty+1$ for all $Q \in \mathbb{R}^{|\mathcal{S}||\mathcal{A}|}$.
        \item The random process $\{w_k^i\}$ is measurable with respect to $\mathcal{F}_{k+1}$, and satisfies $\mathbb{E}[w_k^i \mid \mathcal{F}_k]=0$. In addition, we have $\vert w_k^i \vert\leq 2/(1-\gamma)$ almost surely for all $k\geq 0$.
    \end{enumerate}
\end{lemma}
    \begin{proof}[Proof of Lemma \ref{lemma:property_for_SA}]
    \begin{enumerate}[(1)]
        \item The contraction mapping follows from standard results in MDP theory \cite{puterman2014markov}. Now, for any $Q\in\mathbb{R}^{|\mathcal{S}||\mathcal{A}|}$, we have
\begin{align*}
    \Vert F^i(Q) \Vert_\infty 
    \leq \Vert F^i(Q) -F^i(0)\Vert_\infty+\Vert F^i(0) \Vert_\infty
    \leq \gamma \Vert Q \Vert_\infty +1,
\end{align*}
where the last inequality follows from the fact that $F^i(\cdot)$ is a contraction mapping and $\Vert F^i(0) \Vert_\infty\leq \max_{(s^i,a^i)\in\mathcal{S}^i\times\mathcal{A}^i} \mathcal{R}^i(s^i,a^i)\leq 1$. 
\item The fact that $w_k^i$ is measurable with respect to $\mathcal{F}_k$ follows from the definition of $w_k^i$. We next show that $w_k^i$ is conditionally mean zero. For IQL, we have
\begin{align*}
    \mathbb{E}[w^i_k\mid \mathcal{F}_k]&=\mathbb{E} \left[\mathcal{R}^i(S_k^i,A_k^i)+\gamma \max_{\Bar{a}} Q_k^i(h_1(S_{k+1}),h_2(\Bar{a})) \,\middle|\, \mathcal{F}_k\right]-[F^i(\Phi^i Q_k^i)](S_k,A_k)]\\
    &= \mathcal{R}^i(S_k^i,A_k^i)+ \gamma\mathbb{E}_{\Bar{s}\sim P_{A_k} (S_k,\cdot)} \left[\max_{\Bar{a}}Q_k^i(\Bar{s},\Bar{a})\right]-[F^i(\Phi^i Q_k^i)](S_k,A_k)\\
    &=0,
\end{align*}
where the second equality follows from the Markov property and the tower property of conditional expectations. The proof for $\mathbb{E}[w^i_k\mid \mathcal{F}_k]=0$ in the case of ITD follows from an identical approach.

Finally, using the boundedness of $Q$-functions in finite MDPs \cite{gosavi2006boundedness}, i.e., $\Vert Q_k^i\Vert_\infty \leq 1/(1-\gamma)$, we have
\begin{align*}
    \vert w^i_k \vert\leq 2\mathcal{R}^i(S_k^i,A_k^i) + 2\gamma \Vert Q_k^i \Vert _\infty\leq 2+2\gamma/(1-\gamma)=2/(1-\gamma).
\end{align*}
    \end{enumerate}
\end{proof}

The results in Lemma \ref{lemma:property_for_SA}, together with Assumption \ref{assum_markov_chain}, provide the conditions required for applying \cite[Theorem 3.1]{lin2021multi}. Next, we present the convergence results of IQL and ITD. We will show that the stochastic approximation in Eq. (\ref{eq:IQL_ITD_SA}) converges to the unique fixed point of $F^i(\cdot)$, denoted by $\Tilde{Q}_*^i$ for IQL and $\Tilde{Q}^i$ for ITD. Specifically,  
$\Tilde{Q}_*^i$ is the unique solution to the following equation:
\begin{align*}
    \Pi_1 F^i(\Phi^ix)=x,
\end{align*}
where $\Pi_1:=((\Phi^i)^\mathsf{T} D_1 \Phi^i)^{-1}(\Phi^i)^\mathsf{T}D_1$
with $D_1\in\mathbb{R} ^{|\mathcal{S}||\mathcal{A}|\times |\mathcal{S}||\mathcal{A}|} :=\text{diag}(d_{\pi_b})$ and $F^i(\cdot)$ is defined in Eq.~\eqref{eq:IQL_operator}. As for ITD, $\Tilde{Q}^i$ is the unique solution to the equation
\begin{align*}
    \Pi_2 F^i(\Phi^i x)=x,
\end{align*}
where $\Pi_2:=((\Phi^i)^\mathsf{T}D_2\Phi^i)^{-1}(\Phi^i)^\mathsf{T} D_2$ with
$D_2 \in\mathbb{R} ^{|\mathcal{S}||\mathcal{A}|\times |\mathcal{S}||\mathcal{A}|} :=\text{diag}(d_{\pi_{(t)}})$ and $F^i(\cdot)$ is defined in Eq.~\eqref{eq:ITD_operator}. Note that $\Tilde{Q}^i$ and $\Tilde{Q}_*^i$ are well defined because the operator $\Pi_1 F^i(\Phi^i \cdot)$ and $\Pi_2 F^i(\Phi^i \cdot)$ are also contraction mappings with respect to $\|\cdot\|_\infty$ \cite[Proposition C.1]{lin2021multi}.
The convergence results of IQL and ITD are stated below.

\begin{theorem}\label{thm:IQL_ITD_convergence}
    Suppose that Assumption \ref{assum_markov_chain} is satisfied, and $\alpha_k=\frac{\alpha}{k+k_0}$ with $k_0=\max (4\alpha, 2M_2\log K)$ and $\alpha\geq \frac{2}{\sigma'(1-\gamma)}$. Then, we have the following convergence bounds.
    \begin{enumerate}[(1)]
        \item For $\{ Q_k \}_{k\geq 0}$ generated by IQL (cf. Algorithm \ref{algorithm:IQL}), we have with probability at least $1-\delta'/n$ that
        \begin{align*}
            \left\| Q_K^i-\Tilde{Q}_*^i\right\|_\infty\leq \frac{C'_a}{\sqrt{K+k_0}}+\frac{C_b}{K+k_0}=\Tilde{\mathcal{O}}\left( \frac{1}{\sqrt{K}} \right),
        \end{align*}
        \item For $\{ Q_k \}_{k\geq 0}$ generated by ITD (cf. Algorithm \ref{algorithm:INAC_inner}), we have with probability at least $1-\delta'/n$ that
        \begin{align*}
        \left\Vert Q_K^i-\Tilde{Q}^i \right\Vert_\infty \leq \frac{C'_a}{\sqrt{K+k_0}}+\frac{C_b}{K+k_0}=\Tilde{\mathcal
        O}\left(\frac{1}{\sqrt{K}}\right),
        \end{align*}
    \end{enumerate}
    where $C_a'$ and $C_b$ were defined in Theorem \ref{thm_IQL}.
\end{theorem}
From Theorem~\ref{thm:IQL_ITD_convergence}, we see that both IQL and ITD converge, which is not a straightforward result for IL. In the following sections, we will use the results in Theorem~\ref{thm:IQL_ITD_convergence} to continue the proof of Theorem~\ref{thm_IQL} and Theorem~\ref{thm_INAC}.

\subsection{Proof of Theorem \ref{thm_IQL}}
Let $Q_k(s,a)=\sum_{i\in[n]}Q_k^i(s^i,a^i)$ for all $(s,a)\in \mathcal{S}\times\mathcal{A}$. Denote the optimal $Q$-function of agent $i$ under the separable MDP model $\hat{\mathcal{M}}$ as $\hat{Q}_{\hat{\pi}_*}^i$.

\begin{theorem}\label{thm_IQL_approxi_converge} 
Consider $\{Q_k\}_{k\geq 0}$ generated by Algorithm \ref{algorithm:IQL}. Suppose that Assumption~\ref{assum_markov_chain} is satisfied, and $\alpha_k=\frac{\alpha}{k+k_0}$ with $k_0=\max (4\alpha, 2M_2\log K)$ and $\alpha\geq \frac{2}{\sigma'(1-\gamma)}$. Then, with probability at least $1-\delta'$, we have
        $$\left\Vert Q_{K}-\hat{Q}_{\hat{\pi}_*}\right\Vert_\infty\leq \frac{nC'_a}{\sqrt{K+k_0}}+\frac{nC_b}{K+k_0}+\frac{2n\gamma \mathcal{E}}{(1-\gamma)^2},$$
where the terms $C'_a, C_b$ were defined in Theorem~\ref{thm_IQL}.
\end{theorem}
Instead of directly considering the approximate convergence to the optimal $Q$- function of $\pi_*$ on the original MDP $\mathcal{M}$, we first bound the difference between $Q_K$ and the optimal $Q$-function on the separable MDP $\hat{\mathcal{M}}$ in Theorem \ref{thm_IQL_approxi_converge}. Our approach follows the roadmap illustrated in Figure \ref{fig:analysis_idea} to overcome the challenge of studying the convergence of IQL in the original MDP.
\begin{proof}[Proof of Theorem \ref{thm_IQL_approxi_converge}]
We first apply Theorem \ref{thm:IQL_ITD_convergence} to obtain
\begin{align*}
    \left \Vert Q_K^i-\Tilde{Q}_*^i \right\Vert_\infty\leq \frac{C'_a}{\sqrt{K+k_0}}+\frac{C_b}{K+k_0}=\Tilde{O}\left( \frac{1}{\sqrt{K}} \right)
\end{align*}
with probability at least $1-\delta'/n$. It remains to bound the difference between $\Phi^i \Tilde{Q}_*^i$ and $\hat{Q}^i_{\hat{\pi}_*}$.
\begin{lemma}[Proof in Appendix \ref{proof_of_lemma_IQL_gap_fixedpoint_real}]\label{lemma:IQL_gap_fixedpoint_real}
    It holds that $\| \Phi^i \Tilde{Q}_*^i-\hat{Q}_{\hat{\pi}_*}^i \|_\infty \leq \frac{2\gamma\mathcal{E}}{(1-\gamma)^2}$.
\end{lemma}

Combining the above results, we have with probability at least $1-\delta'$ that
\begin{subequations}\label{eq:bound_Q_K_opt_Q_pi}
    \begin{align}
    \left \Vert Q_K-\hat{Q}_{\hat{\pi}_*} \right \Vert_\infty
    & = \max_{s,a} \left\vert \sum_{i\in[n]} Q^i_K(s^i,a^i)-\sum_{i\in[n]}\hat{Q}_{\hat{\pi}_*}^i(s,a) \right\vert \nonumber \\
    & \leq \max_{s,a} \sum_{i\in[n]} \left\vert Q^i_K(s^i,a^i)-\hat{Q}_{\hat{\pi}_*}^i(s,a) \right\vert \nonumber \\
    & \leq \max_{s,a} \sum_{i\in[n]} \left( \left\vert Q^i_K(s^i,a^i)-\Tilde{Q}_*^i(s^i,a^i) \right\vert+\left\vert (\Phi^i \Tilde{Q}_*^i)(s,a)-\hat{Q}_{\hat{\pi}_*}^i(s,a) \right\vert \right) \label{optqkq_1}\\
    & \leq \sum_{i\in[n]}\left( \left\Vert Q^i_K-\Tilde{Q}_*^i \right\Vert_\infty+\left\Vert (\Phi^i \Tilde{Q}_*^i)-\hat{Q}_{\hat{\pi}_*}^i \right\Vert_\infty \right)\nonumber\\
    & \leq \sum_{i\in[n]} \left( \frac{C'_a}{\sqrt{K+k_0}}+\frac{C_b}{K+k_0}+\frac{2\gamma \mathcal{E}}{(1-\gamma)^2}\right)\label{optqkq_2} \\
    &= \frac{nC'_a}{\sqrt{K+k_0}}+\frac{nC_b}{K+k_0}+\frac{2n\gamma \mathcal{E}}{(1-\gamma)^2}\nonumber,
\end{align}
\end{subequations}
where Eq. (\ref{optqkq_1}) follows from the fact that $(\Phi^i\Tilde{Q}_*^i)(s,a)=\Tilde{Q}_*^i(s^i,a^i)$ and Eq. (\ref{optqkq_2}) follows from Lemma \ref{lemma:IQL_gap_fixedpoint_real} and Theorem \ref{thm:IQL_ITD_convergence}. The probability $1-\delta'$ comes from the union bound.
\end{proof}

Finally, to characterize the difference between the optimal $Q$-functions of the original MDP and the separable one, we need the following lemma. 
\begin{lemma}[Proof in Appendix \ref{proof_of_bound_optimality_gap}]\label{lemma:bound_optimality_gap}
Let $Q_{\pi_*}$ and $\hat{Q}_{\hat{\pi}_*}$ be the optimal Q-functions under the original MDP model $\mathcal{M}$ and the separable one $\hat{\mathcal{M}}$. Then, we have 
    \begin{align*}
        \left \Vert Q_{\pi_*}- \hat{Q}_{\hat{\pi}_*} \right \Vert_\infty \leq \frac{2n\gamma\mathcal{E}}{(1-\gamma)^2}.
    \end{align*}
\end{lemma}

Using Lemma \ref{lemma:bound_optimality_gap}, we have with probability at least $1-\delta'$ that
\begin{align*}
    \Vert Q_K-Q_{\pi_*}\Vert_\infty \leq \Vert Q_K-\hat{Q}_{\hat{\pi}_*}\Vert_\infty+\Vert Q_{\pi_*}-\hat{Q}_{\hat{\pi}_*}\Vert _\infty  \leq \frac{nC'_a}{\sqrt{K+k_0}}+\frac{nC_b}{K+k_0}+\frac{4n\gamma \mathcal{E}}{(1-\gamma)^2}.
\end{align*}
To proceed and obtain the optimality gap in terms of policies, recall that we defined $\pi_K$ as the policy greedily induced by $Q_K$. Denote $a_{K,s}=\max_a Q_K(s,a)$ and $a_{*,s}=\max_{a} Q_{\pi_*}(s,a)$. When $a_{K,s}\neq a_{*,s}$, using the previous inequality, we have
\begin{align*}
    \vert Q_{\pi_*}(s,a_{K,s})-Q_{\pi_*}(s,\pi_{*,s}) \vert \leq \frac{2nC'_a}{\sqrt{K+k_0}}+\frac{2nC_b}{K+k_0} +\frac{8n\gamma\mathcal{E}}{(1-\gamma)^2}.
\end{align*}
Define the advantage function as $A_{\pi}(s,a)=Q_{\pi}(s,a)-V_{\pi}(s)$ for any policy $\pi$ and $(s,a)\in \mathcal{S}\times\mathcal{A}$. Then, we have
\begin{align*}
    \vert A_{\pi_*}(s,a_{K,s})\vert &=\vert Q_{\pi_*}(s,a_{K,s})-V_{\pi_*}(s)\vert\\
    & = \vert Q_{\pi_*}(s,a_{K,s})-Q_{\pi_*}(s,a_{*,s})\vert \\
    & \leq \frac{2nC'_a}{\sqrt{K+k_0}}+\frac{2nC_b}{K+k_0} +\frac{8n\gamma\mathcal{E}}{(1-\gamma)^2}.
\end{align*}
Using the performance difference lemma \cite{kakade2002approximately}, with probability at least $1-\delta'$, we have for any initial state distribution $\mu$ that
\begin{align*}
    V_{\pi_*}^\mu-V_{\pi_K}^\mu  = \frac{1}{1-\gamma} \mathbb{E}_{s\sim o_{\pi_K}^\mu}  A_{\pi_*}(s,a_{K,s})\leq \frac{1}{1-\gamma}\left(\frac{2nC'_a}{\sqrt{K+k_0}}+\frac{2nC_b}{K+k_0}\right) +\frac{8n\gamma\mathcal{E}}{(1-\gamma)^3},
\end{align*}
where $o_{\pi_K}^\mu$ is the occupancy measure defined as  $o_{\pi_K}^\mu (s)=(1-\gamma)\mathbb{E}_{\pi_K}[\sum_{k=0}^\infty\gamma^t \mathbb{P}(S_k=s \mid S_0\sim\mu)]$ for all $s\in \mathcal{S}$. The proof is complete.

\subsection{Proof of Theorem \ref{thm_INAC}}

To prove Theorem \ref{thm_INAC}, we first by analyzing the convergence error of the critic in Appendix \ref{subsec:proof_critic}. Then we analyze the convergence rate of the actor in Appendix \ref{subsec:proof_actor}. Finally, we combine the analysis of the actor and the critic to finish the proof.

\subsubsection{Analysis of the Critic}\label{subsec:proof_critic}
For simplicity of presentation, we first write down only the inner loop of Algorithm \ref{algorithm:INAC} in the following, where we may omit the subscript $t$. The results we derive for the inner loop can be easily combined with the analysis of the outer loop using the Markov property.

\begin{algorithm}[htbp]
	\caption{Inner Loop of Algorithm \ref{algorithm:INAC}}
	\label{algorithm:INAC_inner} 
	\begin{algorithmic}[1]
		\STATE \textbf{Input:} Integer $K$, policy $\pi^i:=\pi_{\theta^i}^i$ from the outer loop, and initialization $Q^i_0=0$.
		\FOR{$k=0,1,2,\cdots,K-1$}
		\STATE Implement $A_k^i\sim \pi^i(\cdot\mid S_k^i)$ (simultaneously with all other agents), and observes $S^i_{k+1}$.
        \STATE Update $Q$-function:
		\STATE          $Q_{k+1}^i(S_k^i,A_k^i)=(1-\alpha_k)Q_k^i(S_k^i,A_k^i)+\alpha_k(\mathcal{R}^i(S_k^i,A_k^i)+\gamma \mathbb{E}_{\Bar{a}^i\sim\pi^i(\cdot\mid S_{k+1}^i)}Q_k^i(S_{k+1}^i,\Bar{a}^i))$
		\ENDFOR
	\end{algorithmic}
\end{algorithm} 
Next, we provide the approximate convergence of the critic. 
\begin{theorem}\label{thm_TD_learning} 
Consider $\{Q_k\}_{k \geq 0}$ generated by Algorithm \ref{algorithm:INAC_inner} with input policy $\pi=(\pi^1,\pi^2,\cdots,\pi^n)$. Suppose that Assumption \ref{assum_markov_chain} is satisfied, and $\alpha_k=\frac{\alpha}{k+k_0}$ with $k_0=\max (4\alpha, 2M_2\log K)$ and $\alpha\geq \frac{2}{\sigma'(1-\gamma)}$. Then, with probability at least $1-\delta'$, we have
        $$\left\Vert Q_{K}-\hat{Q}_{\pi}\right\Vert_\infty\leq \frac{nC'_a}{\sqrt{K+k_0}}+\frac{nC_b}{K+k_0}+\frac{2n\gamma \mathcal{E}}{(1-\gamma)^2},$$
where the constants $C'_a$ and $C_b$ were defined in Theorem \ref{thm_IQL}.
\end{theorem}

\begin{proof}[Proof of Theorem \ref{thm_TD_learning}]\label{proof_of_thm_TD_learning} 
We have verified that ITD is a special case of the stochastic approximation algorithm with state aggregation. Therefore, apply Theorem \ref{thm:IQL_ITD_convergence} and we get
\begin{align*}
    \left\Vert Q_K^i-\Tilde{Q}^i \right\Vert_\infty \leq \frac{C'_a}{\sqrt{K+k_0}}+\frac{C_b}{K+k_0}=\Tilde{\mathcal
    O}\left(\frac{1}{\sqrt{K}}\right)
\end{align*}
Our next step is to bound the gap between $\Phi^i\Tilde{Q}^i$ and $\hat
Q_{\pi}^i$. 
\begin{lemma}[Proof in Appendix \ref{proof_of_lemma_ITD_gap_fixedpoint_real}]\label{lemma:ITD_gap_fixedpoint_real}
    It holds that $\| \Phi^i \Tilde{Q}^i-\hat{Q}_{\pi}^i \|_\infty \leq \frac{2\gamma\mathcal{E}}{(1-\gamma)^2}$.
\end{lemma}
It follows from Lemma \ref{lemma:ITD_gap_fixedpoint_real} and Theorem \ref{thm_TD_learning} that with probability at least $1-\delta'$, we have
\begin{subequations}\label{eq:bound_Q_K_Q_pi}
    \begin{align}
    \Vert Q_K-\hat{Q}_{\pi}\Vert_\infty
    & = \max_{s,a} \left\vert \sum_{i\in[n]} Q^i_K(s^i,a^i)-\sum_{i\in[n]}\hat{Q}_{\pi}^i(s,a) \right\vert  \nonumber\\
    & \leq \max_{s,a} \sum_{i\in[n]} \left\vert Q^i_K(s^i,a^i)-\hat{Q}_{\pi}^i(s,a) \right\vert  \nonumber\\
    & \leq \max_{s,a} \sum_{i\in[n]} \left( \left\vert Q^i_K(s^i,a^i)-\Tilde{Q}^i(s^i,a^i) \right\vert+\left\vert (\Phi^i \Tilde{Q}^i)(s,a)-\hat{Q}_{\pi}^i(s,a) \right\vert \right)  \label{qkqpi_1}\\
    & \leq \sum_{i\in[n]}\left( \left\Vert Q^i_K(s^i,a^i)-\Tilde{Q}^i(s^i,a^i) \right\Vert_\infty+\left\Vert (\Phi^i \Tilde{Q}^i)(s,a)-\hat{Q}_{\pi}^i(s,a) \right\Vert_\infty \right)\nonumber\\
    & \leq \sum_{i\in[n]} \left( \frac{C'_a}{\sqrt{K+k_0}}+\frac{C_b}{K+k_0}+\frac{2\gamma \mathcal{E}}{(1-\gamma)^2}\right)\label{qkqpi_2} \\
    &= \frac{nC'_a}{\sqrt{K+k_0}}+\frac{nC_b}{K+k_0}+\frac{2n\gamma \mathcal{E}}{(1-\gamma)^2}\nonumber,
\end{align}
\end{subequations}
where Eq. (\ref{qkqpi_1}) follows from the fact that $(\Phi^i\Tilde{Q}^i)(s,a)=\Tilde{Q}^i(s^i,a^i)$ and Eq. (\ref{qkqpi_2}) follows from Lemma \ref{lemma:ITD_gap_fixedpoint_real} and Theorem \ref{thm_TD_learning}. The probability $1-\delta'$ comes from the union bound.
\end{proof}

\subsubsection{Analysis of the Actor}\label{subsec:proof_actor}
For ease of presentation, we first write down the outer loop of Algorithm \ref{algorithm:INAC} in Algorithm \ref{algorithm:INAC_outer}. 

\begin{algorithm}[htbp]
	\caption{Outer Loop of Algorithm \ref{algorithm:INAC}}
	\label{algorithm:INAC_outer} 
	\begin{algorithmic}[1]
		\STATE \textbf{Input:} Integer $T$ and initialization $\theta_0^i=0$.
		\FOR{$t=0, 1, 2\cdots,T-1$}
        \STATE $\theta_{t+1}^i=\theta_t^i+\eta_t Q_{t,K}^i$
		\ENDFOR
	\end{algorithmic}
\end{algorithm} 

The following theorem characterizes the convergence rates of Algorithm \ref{algorithm:INAC_outer}.
\begin{theorem}\label{thm_NPG}
Consider $\{\pi_{(t)}=(\pi_{(t)}^1, \pi_{(t)}^2, \cdots, \pi_{(t)}^n)\}$ generated by Algorithm \ref{algorithm:INAC_outer}. Suppose that $\eta_t$ satisfies the condition specified in Theorem \ref{thm_INAC}. Then we have
   \begin{align*}
         \left\Vert Q_{\pi_*}-Q_{(t)} \right\Vert_\infty \leq  \frac{4n\gamma^{t-1}}{(1-\gamma)^2} + \frac{2}{1-\gamma}\sum_{j=0}^{t-1} \gamma^{t-j-1}  \Vert Q_{j,K}-\hat{Q}_{(j)} \Vert_\infty
        +\frac{4n\gamma\mathcal{E}}{(1-\gamma)^2}.
    \end{align*}
\end{theorem}
\begin{proof}[Proof of Theorem \ref{thm_NPG}]
    We begin by decomposing the optimality gap in the following way:
\begin{align}\label{decomposition_of_V_error}
     \Vert Q_{\pi_*}-Q_{(t)}  \Vert_\infty  \leq \underbrace{\Vert \hat{Q}_{\hat{\pi}_*}-\hat{Q}_{(t)}\Vert_\infty }_{v_1}+\underbrace{ \Vert Q_{\pi_*}-\hat{Q}_{\hat{\pi}_*}\Vert_\infty}_{v_2}+\underbrace{ \Vert \hat{Q}_{(t)}-Q_{(t)}\Vert_\infty }_{v_3},
\end{align}
where we recall that $\hat{\pi}_*$ denotes the optimal policy of the separable MDP $\hat{\mathcal{M}}$. 
On the RHS of Eq. (\ref{decomposition_of_V_error}), the term $v_1$ captures the optimality gap of the output of Algorithm \ref{algorithm:INAC_outer} with respect to the separable MDP model $\hat{\mathcal{M}}$, and the terms $v_2$ and $v_3$ are both induced by the model difference between the original MDP $\mathcal{M}$ and the separable one $\hat{\mathcal{M}}$.

To bound the term $v_1$, since $\hat{\mathcal{M}}$ is a separable MDP, when each agent implements Algorithm \ref{algorithm:INAC_outer}, the update equation for the global policy parameter $\theta_t\in\mathbb{R}^{|\mathcal{S}||\mathcal{A}|}$ can be written as
\begin{align}\label{eq:NPGtogether}
    \theta_{t+1}=\theta_t+\eta_t Q_{t,K},
\end{align}
where $Q_{t,K}(s,a)=\sum_{i=1}^nQ_{t,K}^i(s^i,a^i)$ for all $s=(s^1,\cdots,s^n)\in \mathcal{S}$ and $a=(a^1,\cdots,a^n)\in \mathcal{A}$.
Therefore, we can use existing results on single-agent natural actor-critic to bound the term $v_1$. In particular, Theorem 2.1 in \cite{chen2023approximate} provides us the following result.
\begin{lemma} \label{lemma_existing_NPG_result}
Consider $\{\pi_{(t)}=(\pi_{(t)}^1,\pi_{(t))}^2,,\cdots,\pi_{(t)}^n)\}_{t\geq 0}$ generated by each agent implementing Algorithm \ref{algorithm:INAC_outer}. When the stepsize $\eta_t$ satisfies $\eta_{t}\geq \log (\frac{1}{\min_s \pi_{(t)}(a_{t,s}\mid s)})/\gamma^{2t-1}$, where $a_{t,s}=\arg\max_a Q_{t,K}(s,a)$, we have
    \begin{align*}
        \left \Vert \hat{Q}_{\hat{\pi}_*}-\hat{Q}_{(t)}\right \Vert_\infty  \leq  \gamma^t \Vert \hat{Q}_{\pi_*}-\hat{Q}_{(0)}\Vert_\infty+ \frac{2\gamma}{1-\gamma}\sum_{j=0}^{t-1} \gamma^{t-j-1}  \Vert Q_{j,K}-\hat{Q}_{(j)} \Vert_\infty +\frac{2\gamma^{t-1}}{(1-\gamma)^2}.
    \end{align*}
\end{lemma}
To apply Lemma \ref{lemma_existing_NPG_result}, we need to verify that the stepsizes in Theorem \ref{thm_INAC} satisfy the condition in Lemma \ref{lemma_existing_NPG_result}. Using Eq. (\ref{eq:NPGtogether}), we have for any state-action pair $(s,a)$ that
\begin{align*}
    (\theta_{t})_{s,a}=(\theta_{t-1})_{s,a}+\eta_{t-1} Q_{t-1,K}(s,a)= (\theta_{0})_{s,a}+\sum_{j=0}^{t-1} \eta_j Q_{j,K}(s,a) \leq \sum_{j=0}^{t-1} \eta_j \frac{n}{1-\gamma}=\Bar{Q} \sum_{j=0}^{t-1} \eta_j,
\end{align*}
where $\Bar{Q}:=\frac{n}{1-\gamma}$ and the last inequality follows from the boundedness of the $Q$-functions in finite MDPs \cite{gosavi2006boundedness}, that is,
\begin{align*}
    Q_{t,K}(s,a)=\sum_{i=1}^nQ_{t,K}^i(s^i,a^i)\leq \sum_{i=1}^n\|Q_{t,K}^i\|_\infty\leq n/(1-\gamma).
\end{align*}
The previous inequality allows us to derive a lower bound for $\pi_{(t)}(a \mid s)$:
\begin{align*}
    \pi_{(t)}(a \mid s) & = \frac{\exp\{(\theta_{t})_{s,a}\}}{\sum_{a'\in \mathcal{A}} \exp\{(\theta_{t})_{s,a'}\}}\\
    & = \frac{1}{1+\sum_{a'\neq a} \exp\{(\theta_{t})_{s,a'}-(\theta_{t})_{s,a}\}}\\
    & \geq \frac{1}{1+\sum_{a'\neq a} \exp\{2\Bar{Q} \sum_{j=0}^{t-1} \eta_j \}} \\
    & \geq \frac{1}{\vert \mathcal{A} \vert \exp\{2\Bar{Q} \sum_{j=0}^{t-1} \eta_j \}}.
\end{align*}
It follows that
\begin{align*}
    \log \left(\frac{1}{\min_s \pi_{(t)}(a_{t,s}\mid s)}\right)/\gamma^{2t-1}\leq 2\Bar{Q}\sum_{j=0}^{t-1} \eta_j \log \vert \mathcal{A} \vert/\gamma^{2t-1}.
\end{align*}
Therefore, the condition on the stepsizes in Lemma \ref{lemma_existing_NPG_result} is satisfied. Now, using Lemma \ref{lemma_existing_NPG_result} for the term $v_1$ on the RHS of Eq. (\ref{decomposition_of_V_error}), we have
\begin{align*}
    v_1 & =  \Vert \hat{Q}_{\hat{\pi}_*}-\hat{Q}_{(t)} \Vert_\infty \\
    & \leq \gamma^t \Vert \hat{Q}_{\pi_*}-\hat{Q}_{(0)}\Vert_\infty+ \frac{2\gamma}{1-\gamma}\sum_{j=0}^{t-1} \gamma^{t-j-1} \Vert Q_{j,K}-\hat{Q}_{(j)} \Vert_\infty+\frac{2\gamma^{t-1}}{(1-\gamma)^2}\\
    & \leq \frac{4n\gamma^{t-1}}{(1-\gamma)^2} + \frac{2}{1-\gamma}\sum_{j=0}^{t-1} \gamma^{t-j-1}  \Vert Q_{j,K}-\hat{Q}_{(j)} \Vert_\infty,
\end{align*}
where the last step follows from:
$$\Vert \hat{Q}_{\pi_*}-\hat{Q}_{(0)}\Vert_\infty \leq \|\hat{Q}_{\pi_*}\|_\infty+\|\hat{Q}_{(0)}\|_\infty \leq \frac{2n}{1-\gamma}.$$
We next consider the terms $v_2$ and $v_3$ in Eq. (\ref{decomposition_of_V_error}). Using Lemma \ref{lemma:bound_optimality_gap}, we bound the term $v_2$ in Eq. (\ref{decomposition_of_V_error}) as
\begin{align*}
    v_2 = \Vert Q_{\pi_*}-\hat{Q}_{\hat{\pi}_*} \Vert_\infty\leq \frac{2n\gamma\mathcal{E}}{(1-\gamma)^2}.
\end{align*}

To bound $v_3$, we need the following lemma, which bounds the gap between $Q$-functions of the same policy applied to the original MDP and the separable one.
\begin{lemma}[Proof in Appendix \ref{proof_of_bound_model_difference}]\label{bound_model_difference}
Given a policy $\pi$, let $Q_{\pi}$ and $\hat{Q}_{\pi}$ be the $Q$-functions of the original MDP $\mathcal{M}$ and the separable MDP $\hat{\mathcal{M}}$. Then, we have
    \begin{align*}
        \Vert Q_{\pi}-\hat{Q}_{\pi} \Vert_\infty\leq \frac{2n\gamma\mathcal{E}}{(1-\gamma)^2}.
    \end{align*}
\end{lemma}
Using Lemma \ref{bound_model_difference}, we have
\begin{align*}
    v_3 = \Vert Q_{(t)}-\hat{Q}_{(t)} \Vert_\infty\leq \frac{2n\gamma\mathcal{E}}{(1-\gamma)^2}.
\end{align*}
Theorem \ref{thm_NPG} follows from using the upper bounds we derived for the terms $v_1$, $v_2$, and $v_3$ in Eq. (\ref{decomposition_of_V_error}).
\end{proof}

In light of the analyses for the actor and the critic, to finish proving Theorem \ref{thm_INAC}, we combine the bounds in Theorem \ref{thm_TD_learning} and Theorem \ref{thm_NPG}.
Specifically, for any $j \leq t-1 $, with probability at least $1-\delta'$, we have
\begin{align*}
    \Vert Q_{j,K}-\hat{Q}_{(j)} \Vert_\infty
    \leq \frac{nC'_a}{\sqrt{K+k_0}}+\frac{nC_b}{K+k_0}+\frac{2n\gamma \mathcal{E}}{(1-\gamma)^2}.
\end{align*}
Let $\delta'=\delta/T$ in the above result. Using union bound, we have with probability at least $1-\delta$ that
\begin{align*}
    \Vert Q_{j,K}-\hat{Q}_{(j)} \Vert_\infty
    \leq \frac{nC_a}{\sqrt{K+k_0}}+\frac{nC_b}{K+k_0}+\frac{2n\gamma \mathcal{E}}{(1-\gamma)^2},\;\forall\,j\leq t-1.
\end{align*}
Therefore with probability at least $1-\delta$, we have
\begin{align*}
    \sum_{j=0}^{T-1} \gamma^{T-j-1} \Vert Q_{j,K} -\hat{Q}_{\{j\}}\Vert_\infty&\leq \sum_{j=0}^{T-1} \gamma^{T-j-1} \left(\frac{nC_a}{\sqrt{K+k_0}}+\frac{nC_b}{K+k_0}+\frac{2n\gamma \mathcal{E}}{(1-\gamma)^2}\right)\\
    & \leq \frac{1}{1-\gamma}\left(\frac{nC_a}{\sqrt{K+k_0}}+\frac{nC_b}{K+k_0}+\frac{2n\gamma \mathcal{E}}{(1-\gamma)^2}\right).
\end{align*}
Applying the bound in the above inequality to Theorem \ref{thm_NPG} completes the proof of Theorem \ref{thm_INAC}. 

\subsection{Proof of All Supporting Lemmas}
\subsubsection{Proof of Lemma \ref{lemma:IQL_gap_fixedpoint_real}}
\label{proof_of_lemma_IQL_gap_fixedpoint_real}
Given a positive definite matrix $D$, define its corresponding weighted $\ell_2$-norm as $\Vert x \Vert_{D}=(x^\mathsf{T} Dx)^{1/2}$. According to the definition of $\Pi_1$, $\Phi^i \Pi_1$ is the projection matrix that projects a vector in $\mathbb{R}^{|\mathcal{S}||\mathcal{A}|}$ to the set $\{\Phi^i r \mid r\in \mathbb{R}^{|\mathcal{S}^i||\mathcal{A}^i|}\}$ with respect to $\|\cdot\|_{D_1}$. Therefore, for any $Q\in \mathbb{R}^{|\mathcal{S}||\mathcal{A}|}$, we have:
\begin{align}\label{eq:definition_of_projection}
    \Phi^i\Pi_1 Q={\arg\min}_{q\in\{\Phi^i r\mid r\in \mathbb{R}^{|\mathcal{S}^i||\mathcal{A}^i|}\}} \Vert Q-q \Vert_{D_1}.
\end{align}
We first show that $\Phi^i\Pi_1(\cdot)$ is nonexpansive with respect to $\|\cdot\|_\infty$.
\begin{lemma}\label{lemma:projection_nonexpansive}
    It holds that $\Vert \Phi^i \Pi_1 Q \Vert _\infty \leq \Vert Q \Vert_\infty$ for any $Q\in \mathbb{R}^{|\mathcal{S}||\mathcal{A}|}$.
\end{lemma}
\begin{proof}[Proof of Lemma \ref{lemma:projection_nonexpansive}]
From Eq.~\eqref{eq:definition_of_projection}, we have
\begin{align*}
    \Pi_1 Q=r_*:={\arg\min}_{r\in \mathbb{R}^{|\mathcal{S}^i||\mathcal{A}^i|}}\sum_{(s,a)\in\mathcal{S}\times\mathcal{A}} d_{\pi_b}(s,a)\big(Q(s,a)-r(h_1(s),h_2(a))\big).
\end{align*}
For any $ (\Bar{s},\Bar{a})$ in $\mathcal{S}\times \mathcal{A}$, the value of $r_*(\Bar{s}^i,\Bar{a}^i)$ must satisfy
\begin{align*}
    \min_{s\in h_1^{-1}(\Bar{s}^i),a\in h_2^{-1}(\Bar{a}^i)} Q(s,a)\leq r_*
(\Bar{s}^i,\Bar{a}^i)\leq \max_{s\in h_1^{-1}(\Bar{s}^i),a\in h_2^{-1}(\Bar{a}^i)} Q(s,a),
\end{align*}
which leads to
\begin{align*}
    \left\vert [\Phi^i\Pi_1 Q](\Bar{s},\Bar{a})\right\vert =\left\vert [\Pi_1 Q](\Bar{s}^i,\Bar{a}^i)\right\vert \leq \max_{s\in h_1^{-1}(\Bar{s}^i),a\in h_2^{-1}(\Bar{a}^i)} \left \vert Q(s,a) \right \vert.
\end{align*}
Therefore, we have $\Vert \Phi^i \Pi_1 Q \Vert _\infty \leq \Vert Q \Vert_\infty$.
\end{proof}
To proceed, recall the definition of $\hat{Q}_\pi^i$:
\begin{align*}
    \hat{Q}_\pi^i(s,a)= \hat{\mathbb{E}}_{\pi} \left[\sum_{k=0}^\infty \gamma^k\mathcal{R}^i(S^i_k,A^i_k)\;\middle|\;S_0=s,A_0=a\right],\quad \forall \;(s,a)\in\mathcal{S}\times \mathcal{A},
\end{align*}
where $\hat{\mathbb{E}}_\pi[\,\cdot\,]$ denotes the expectation with respect to the separable MDP transition kernel $\hat{\mathcal{P}}$. Observe that $\hat{Q}_\pi^i(s_1,a_1)=\hat{Q}_\pi^i(s_2,a_2)$ when $(s_1,a_1), (s_2,a_2)\in \mathcal{S}\times \mathcal{A}$ and $(s_1^i,a_1^i)=(s_2^i,a_2^i)$. Therefore, it follows from Lemma~\ref{lemma:projection_nonexpansive} that $\Phi^i\Pi_1 \hat{Q}_\pi^i=\hat{Q}_\pi^i$. In addition, since
\begin{align*}
    \Vert \Phi^i \Tilde{Q}_*^i-\hat{Q}_{\hat{\pi}_*}^i \Vert_\infty 
    &= \Vert \Phi^i \Tilde{Q}_*^i-\Phi^i \Pi_1 \hat{Q}_{\hat{\pi}_*}^i \Vert_\infty\notag\\
    &= \Vert \Phi^i \Pi_1 F_*^i(\Phi^i \Tilde{Q}_*^i)-\Phi^i \Pi_1 \hat{Q}_{\hat{\pi}_*}^i \Vert_\infty \tag{$\Pi_1 F_*^i(\Phi^i \Tilde{Q}_*^i)=\Tilde{Q}_*^i$}\\
    &\leq \Vert F_*^i(\Phi^i \Tilde{Q}_*^i)-\hat{Q}_{\hat{\pi}_*}^i \Vert_\infty \tag{$\Phi^i\Pi_1$ is nonexpansive}\\
    & \leq \Vert F_*^i(\Phi^i \Tilde{Q}_*^i)-F_*^i(\hat{Q}_{\hat{\pi}_*}^i) \Vert_\infty+\Vert F_*^i(\hat{Q}_{\hat{\pi}_*}^i)-\hat{Q}_{\hat{\pi}_*}^i \Vert_\infty \tag{Triangle inequality}\\
    & \leq \gamma \Vert \Phi^i \Tilde{Q}_*^i-\hat{Q}_{\hat{\pi}_*}^i \Vert_\infty+\Vert F_*^i(\hat{Q}_{\hat{\pi}_*}^i)-\hat{Q}_{\hat{\pi}_*}^i \Vert_\infty\tag{$F_*^i$ is a $\gamma$-contraction},
\end{align*}
we have
\begin{align*}
    \Vert \Phi^i \Tilde{Q}_*^i-\hat{Q}_{\hat{\pi}_*}^i \Vert_\infty
    \leq \; & \frac{1}{1-\gamma}\Vert F_*^i(\hat{Q}_{\hat{\pi}_*}^i)-\hat{Q}_{\hat{\pi}_*}^i \Vert_\infty\\
    = \; & \frac{1}{1-\gamma} \max_{(s,a)} \left\vert \gamma \mathbb{E}_{\Bar{s}\sim P_a(s,\cdot)}\max_{\Bar{a}} \hat{Q}_{\hat{\pi}_*}^i(\Bar{s},\Bar{a})-\gamma\mathbb{E}_{\Bar{s}\sim \hat{P}_a(s,\cdot)}\max_{\Bar{a}} \hat{Q}_{\hat{\pi}_*}^i(\Bar{s},\Bar{a}) \right\vert\\
    \leq \; & \frac{\gamma}{1-\gamma}\max_{(s,a)}\sum_{\Bar{s}\in\mathcal{S}}\left\vert P_a(s,\Bar{s})-\hat{P}_a(s,\Bar{s}) \right\vert \max_{\Bar{a}\in\mathcal{A}}\hat{Q}_{\hat{\pi}_*}^i(\Bar{s},\Bar{a})\\
    \leq \; & \frac{\gamma}{(1-\gamma)^2}\max_{(s,a)}\sum_{\Bar{s}\in\mathcal{S}}\left\vert P_a(s,\Bar{s})-\hat{P}_a(s,\Bar{s}) \right\vert\\
    \leq \; & \frac{2\gamma\mathcal{E}}{(1-\gamma)^2},
\end{align*}
where the last two inequalities follow from $\max_{\Bar{a}\in\mathcal{A}}\hat{Q}_{\hat{\pi}_*}^i(\Bar{s},\Bar{a})\leq \Vert \hat{Q}_{\hat{\pi}_*}^i \Vert_\infty\leq \frac{1}{1-\gamma}$ and the definition of dependence level $\mathcal{E}$.

\subsubsection{Proof of Lemma \ref{lemma:bound_optimality_gap}}\label{proof_of_bound_optimality_gap}
Let $\mathcal{T}_*:\mathbb{R}^{|\mathcal{S}||\mathcal{A}|}\mapsto \mathbb{R}^{|\mathcal{S}||\mathcal{A}|}$ be the Bellman optimality operator of the MDP model $\mathcal{M}$ defined as
\begin{align*}
    [\mathcal{T}_* Q](s,a)=\mathcal{R}(s,a)+\gamma \mathbb{E}_{s'\sim P_a(s,\cdot)}\left[\max_{a'} Q(s',a')\right],\quad \forall\;(s,a).
\end{align*}
The Bellman optimality operator $\hat{\mathcal{T}}_*$ for the separable MDP model $\hat{\mathcal{M}}$ is defined similarly.

Recall that $ \pi_* $ and $\hat{\pi}_*$ are the optimal policies for model $\mathcal{M}$ and $\hat{\mathcal{M}}$ respectively. Then, we have
\begin{align*}
    \Vert Q_{\pi_*}-\hat{Q}_{\hat{\pi}_*} \Vert_\infty & \leq \Vert Q_{\pi_*}-\mathcal{T}_*\hat{Q}_{\hat{\pi}_*}\Vert_\infty+\Vert \mathcal{T}_*\hat{Q}_{\hat{\pi}_*}-\hat{Q}_{\hat{\pi}_*}\Vert_\infty\\
    &\leq \gamma \Vert Q_{\pi_*}-\hat{Q}_{\hat{\pi}_*} \Vert_\infty+\Vert \mathcal{T}_*\hat{Q}_{\hat{\pi}_*}-\hat{Q}_{\hat{\pi}_*}\Vert_\infty,
\end{align*}
where last line follows from  $\mathcal{T}_*Q_{\pi^*}=Q_{\pi^*}$ and the Bellman optimality operator $\mathcal{T}_*$ being a $\gamma$-contraction mapping with respect to $\|\cdot\|_\infty$. It follows that
\begin{align}\label{bound_optimality_gap_mid}
    \left\Vert Q_{\pi_*}-\hat{Q}_{\hat{\pi}_*} \right\Vert_\infty\leq \frac{1}{1-\gamma}\left\Vert \mathcal{T}_*\hat{Q}_{\hat{\pi}_*}-\hat{Q}_{\hat{\pi}_*}\right\Vert_\infty.
\end{align}
Now, for all $(s,a)\in \mathcal{S}\times \mathcal{A}$, we have
\begin{subequations}
\begin{align}
    &\left\vert\hat{Q}_{\hat{\pi}_*}(s,a)-(\mathcal{T}_*\hat{Q}_{\hat{\pi}_*})(s,a)\right\vert \nonumber\\
    \leq & \left\vert\gamma\sum_{s'\in\mathcal{S}} \hat{P}_a(s,s') \max_{a'} \hat{Q}_{\hat{\pi}_*}(s',a')-\gamma\sum_{s'\in\mathcal{S}} P_a(s,s') \max_{a'} \hat{Q}_{\hat{\pi}_*}(s',a')\right\vert\nonumber\\
    \leq & 2\gamma \left \Vert \hat{P}_a(s,\cdot)-P_a(s,\cdot) \right \Vert_{\text{TV}} \frac{n}{1-\gamma}\label{eqtvd1}\\
    \leq & \frac{2n\gamma\mathcal{E}}{1-\gamma} \label{eqtvd2},
\end{align}
\end{subequations}
where Eq.\eqref{eqtvd2} follows from the definition of $\mathcal{E}$, and Eq. \eqref{eqtvd1} follows from the following equivalent definition of the total variation distance \cite{levin2017markov}:
\begin{align*}
    \|\nu_1-\nu_2\|_{\text{TV}}=\frac{1}{2}\sup_{f:\|f\|_\infty\leq 1}\left|\int fd\nu_1-\int fd\nu_2\right|.
\end{align*}
Substituting the above result into Eq. (\ref{bound_optimality_gap_mid}), we have
\begin{align*}
     \left\Vert Q_{\pi_*}-\hat{Q}_{\hat{\pi}_*} \right\Vert_\infty\leq \frac{2n\gamma\mathcal{E}}{(1-\gamma)^2}.  
\end{align*}

\subsubsection{Proof of Lemma \ref{lemma:ITD_gap_fixedpoint_real}} \label{proof_of_lemma_ITD_gap_fixedpoint_real}
Recall the definition of $\Vert \cdot \Vert_D$ in the proof of Lemma \ref{lemma:IQL_gap_fixedpoint_real} in Appendix \ref{proof_of_lemma_IQL_gap_fixedpoint_real}. According to the definition of $\Pi_2$, $\Phi^i \Pi_2$ is the projection matrix that projects a vector in $\mathbb{R}^{|\mathcal{S}||\mathcal{A}|}$ to the set $\{\Phi^i r \mid r\in \mathbb{R}^{|\mathcal{S}^i||\mathcal{A}^i|}\}$. Then, for any $Q\in \mathbb{R}^{|\mathcal{S}||\mathcal{A}|}$, we have the following:
\begin{align*}
    \Phi^i\Pi_2 Q={\arg\min}_{\Bar{Q}\in\{\Phi^i r\mid r\in \mathbb{R}^{|\mathcal{S}^i||\mathcal{A}^i|}\}} \Vert Q-\Bar{Q}\Vert_{D_2}.
\end{align*}
Next, we present a nonexpansive property of the operator $\Phi^i \Pi_2(\cdot)$.
\begin{lemma}\label{lemma:projection_nonexpansive_2}
    It holds that $\Vert \Phi^i \Pi_2 Q \Vert _\infty \leq \Vert Q \Vert_\infty$ for any $Q\in \mathbb{R}^{|\mathcal{S}||\mathcal{A}|}$.
\end{lemma}
The proof of Lemma \ref{lemma:projection_nonexpansive_2} is similar to the proof of Lemma \ref{lemma:projection_nonexpansive}, and is omitted here.
 
Observe that $\hat{Q}_\pi^i(s_1,a_1)=\hat{Q}_\pi^i(s_2,a_2)$ for any policy $\pi$ when $(s_1,a_1), (s_2,a_2)\in \mathcal{S}\times \mathcal{A}$ and $(s_1^i,a_1^i)=(s_2^i,a_2^i)$. Then, it follows from Lemma~\ref{lemma:projection_nonexpansive_2} that $\Phi^i\Pi_2 \hat{Q}_\pi^i=\hat{Q}_\pi^i$.
In addition, note that we have
\begin{subequations}
\begin{align}
    \Vert \Phi^i \Tilde{Q}^i-\hat{Q}_\pi^i \Vert_\infty 
    &= \Vert \Phi^i \Tilde{Q}^i-\Phi^i \Pi_2 \hat{Q}_\pi^i \Vert_\infty\nonumber\\
    &= \Vert \Phi^i \Pi_2 F^i(\Phi^i \Tilde{Q}^i)-\Phi^i \Pi_2 \hat{Q}_\pi^i \Vert_\infty \label{eq:ll2}\\
    &\leq \Vert F^i(\Phi^i \Tilde{Q}^i)-\hat{Q}_\pi^i \Vert_\infty \label{eq:ll3}\\
    & \leq \Vert F^i(\Phi^i \Tilde{Q}^i)-F^i(\hat{Q}_\pi^i) \Vert_\infty+\Vert F^i(\hat{Q}_\pi^i)-\hat{Q}_\pi^i \Vert_\infty \nonumber\\
    & \leq \gamma \Vert \Phi^i \Tilde{Q}^i-\hat{Q}_\pi^i \Vert_\infty+\Vert F^i(\hat{Q}_\pi^i)-\hat{Q}_\pi^i \Vert_\infty\label{eq:ll4},
\end{align}
\end{subequations}
where Eq. (\ref{eq:ll2}) follows from the fact that $\Tilde{Q}^i$ is the solution of equation $\Pi_2 F^i(\Phi^i r)=r$, Eq. \eqref{eq:ll3} follows from Lemma~\ref{lemma:projection_nonexpansive_2}, and the last inequality Eq. (\ref{eq:ll4}) follows from the fact that $F^i(\cdot)$ is a $\gamma$-contraction in $\Vert\cdot\Vert_\infty$. Combined with the Bellman equation, the above inequality leads to
\begin{align*}
    \Vert \Phi^i \Tilde{Q}^i-\hat{Q}_\pi^i \Vert_\infty
    \leq \; & \frac{1}{1-\gamma}\Vert F^i(\hat{Q}_\pi^i)-\hat{Q}_\pi^i \Vert_\infty\\
    = \; & \frac{1}{1-\gamma} \max_{(s,a)} \left\vert \gamma \mathbb{E}_{\Bar{s}\sim P_a(s,\cdot),\Bar{a}\sim \pi (\cdot\mid \Bar{s})} [\hat{Q}_\pi^i(\Bar{s},\Bar{a})]-\gamma\mathbb{E}_{\Bar{s}\sim \hat{P}_a(s,\cdot),\Bar{a}\sim \pi (\cdot\mid \Bar{s})} [\hat{Q}_\pi^i(\Bar{s},\Bar{a})]  \right\vert\\
    \leq \; & \frac{\gamma}{1-\gamma}\max_{(s,a)}\sum_{\Bar{s}\in\mathcal{S}}\left\vert P_a(s,\Bar{s})-\hat{P}_a(s,\Bar{s}) \right\vert \sum_{\Bar{a}\in\mathcal{A}}\pi(\Bar{a}\mid \Bar{s})\hat{Q}_\pi^i(\Bar{s},\Bar{a})\\
    \leq \; & \frac{\gamma}{(1-\gamma)^2}\max_{(s,a)}\sum_{\Bar{s}\in\mathcal{S}}\left\vert P_a(s,\Bar{s})-\hat{P}_a(s,\Bar{s}) \right\vert\\
    \leq \; & \frac{2\gamma\mathcal{E}}{(1-\gamma)^2},
\end{align*}
where the last two inequalities follow from $\sum_{\Bar{a}\in\mathcal{A}}\pi_(\Bar{a}\mid \Bar{s})\hat{Q}_\pi^i(\Bar{s},\Bar{a})\leq \Vert \hat{Q}_\pi^i \Vert_\infty\leq \frac{1}{1-\gamma}$ and the definition of dependence level $\mathcal{E}$.

\subsubsection{Proof of Lemma \ref{bound_model_difference}}\label{proof_of_bound_model_difference}
Let $\mathcal{T}_\pi:\mathbb{R}^{|\mathcal{S}||\mathcal{A}|}\mapsto\mathbb{R}^{|\mathcal{S}||\mathcal{A}|}$ be the Bellman operator associated with policy $\pi$, which is defined as
\begin{align*}
   [\mathcal{T}_\pi Q](s,a)=\mathcal{R}(s,a)+\gamma \mathbb{E}_{s'\sim P_a(s,\cdot), a'\sim \pi(\cdot \mid s')} [Q(s',a')], \quad \forall\, (s,a)\in \mathcal{S}\times\mathcal{A}.
\end{align*}
Similarly, we define $\hat{\mathcal{T}}_\pi(\cdot)$ as the Bellman operator associated with the policy $\pi$ under the separable MDP model $\hat{\mathcal{M}}$. It is well-known that both $\mathcal{T}_\pi$ and $\hat{\mathcal{T}}_\pi$ are  contractive operators with respect to $\|\cdot\|_\infty$, with a common contraction factor $\gamma$ \cite{sutton2018reinforcement}. In addition, $Q_\pi$ (respectively, $\hat{Q}_\pi$) is the unique fixed point of $\mathcal{T}_\pi$ (respectively, $\hat{\mathcal{T}}_\pi$).

To proceed, note that the gap between the Q-functions $Q_{\pi}$ and $\hat{Q}_{\pi}$ can be bounded as
\begin{align*}
    \Vert Q_{\pi}-\hat{Q}_{\pi} \Vert_\infty &=\Vert Q_{\pi}-\mathcal{T}_\pi\hat{Q}_{\pi}+\mathcal{T}_\pi\hat{Q}_{\pi}-\hat{Q}_{\pi} \Vert_\infty\\
    & \leq \Vert Q_{\pi}-\mathcal{T}_\pi\hat{Q}_{\pi}\Vert_\infty+\Vert \mathcal{T}_\pi\hat{Q}_{\pi}-\hat{Q}_{\pi}\Vert_\infty\tag{Triangle inequality}\\
    &= \Vert \mathcal{T}_\pi Q_{\pi}-\mathcal{T}_\pi\hat{Q}_{\pi} \Vert_\infty+\Vert \mathcal{T}_\pi\hat{Q}_{\pi}-\hat{Q}_{\pi}\Vert_\infty\tag{This follows from $\mathcal{T}_\pi Q_{\pi}=Q_{\pi}$.}\\
    &\leq \gamma \Vert Q_{\pi}-\hat{Q}_{\pi} \Vert_\infty+\Vert \mathcal{T}_\pi\hat{Q}_{\pi}-\hat{Q}_{\pi}\Vert_\infty,
\end{align*}
where the last inequality follows from that the Bellman operator $\mathcal{T}_\pi$ is a $\gamma$-contraction mapping with respect to $\|\cdot\|_\infty$. It follows that
\begin{subequations}
\begin{align}
    &\left\Vert Q_{\pi}-\hat{Q}_{\pi} \right\Vert_\infty \nonumber\\
    \leq\; & \frac{1}{1-\gamma}\left\Vert \mathcal{T}_\pi\hat{Q}_{\pi}-\hat{Q}_{\pi}\right\Vert_\infty\nonumber\\
    =\;&\frac{1}{1-\gamma}\left\Vert \mathcal{T}_\pi\hat{Q}_{\pi}-\hat{\mathcal{T}}_\pi\hat{Q}_{\pi}\right\Vert_\infty\nonumber\\
    =\;&\frac{\gamma}{1-\gamma}\max_{(s,a)}\left|\sum_{s'}P_a(s,s')\mathbb{E}_{a'\sim \pi(\cdot \mid s')}[\hat{Q}_{\pi}(s',a')]-\sum_{s'}\hat{P}_a(s,s')\mathbb{E}_{a'\sim \pi(\cdot \mid s')}[\hat{Q}_{\pi}(s',a')]\right|\label{eqeq1}\\
    \leq \;&\frac{2n\gamma}{(1-\gamma)^2}\left\|\hat{P}_a(s,\cdot)-P_a(s,\cdot)\right\|_{\text{TV}}\label{eqeq2}\\
    \leq \;&\frac{2n\gamma\mathcal{E}}{(1-\gamma)^2},\label{eqeq3}
\end{align}
\end{subequations}
where Eq. (\ref{eqeq1}) follows from the definition of $\mathcal{T}_\pi$ and $\hat{\mathcal{T}}_\pi$, Eq. (\ref{eqeq2}) also follows from the equivalent definition of the total variation distance, and Eq. (\ref{eqeq3}) follows from the definition of $\mathcal{E}$.

\subsection{Stochastic Approximation with State Aggregation} \label{SA_with_state_aggregation}
In this section, we present the results of stochastic approximation with state aggregation in \cite{lin2021multi}.

Let $\mathcal{N}=\{1,2,\cdots,n\}$ be the state space of $\{i_k\}$, where $\{i_k\}_{k=0}^\infty$ is the sequence of states visited by a Markov chain. Let $\mathcal{L}=\{1,2,\cdots,l\}, (l\leq n)$ be the abstract state space. The surjection $h: \mathcal{N}\mapsto \mathcal{L}$ is used to convert every state in $\mathcal{N}$ to its abstraction set $\mathcal{L}$. Given parameter $x\in\mathbb{R}^{l}$ and function $F: \mathbb{R}^n\rightarrow \mathbb{R}^l$, we consider the generalized stochastic approximation that updates $x(k)\in \mathbb{R}^l$ starting from $x(0)=0$:
\begin{align}\label{eq:sa_yiheng}
    x_{h(i_k)}(k+1)& =x_{h(i_k)}(k)+\alpha_k\left( F_{i_k}(\Phi x(k))-x_{h(i_k)}(k)+w(k) \right),\\
    x_j(k+1)&=x_j(k), \text{for}\; j\neq h(i_k), j\in \mathcal{L}, 
\end{align}
where the matrix $\Phi\in \mathbb{R}^{n\times l}$ is given by
\begin{align*}
\Phi_{ij}=
    \begin{cases}
        1, & \text{if}\; h(i)=j,\\
        0, & \text{otherwise},
    \end{cases}
    \forall i\in\mathcal{N}, j\in\mathcal{L}.
\end{align*}
Before presenting the convergence result, we first give the assumptions required for the theorem in  \cite{lin2021multi}.
\begin{assumption}\label{assump:egodicity}
    \textit{The random process $\{i_k\}_{k\geq 0}$ is an aperiodic and irreducible Markov chain on the space $\mathcal{S}\times \mathcal{A}$ with stationary distribution $d=\{d_1,d_2,\cdots,d_n\}$. In addition, letting $d'_j=\sum_{i\in h^{-1}(j)} d_i$ for all $j\in \mathcal{L}$ and $\sigma'=\inf_{j\in \mathcal{L}} d'_j$, there exist constants $M_1\geq 0$ and $M_2\geq 1$ such that
    \begin{align}
        \sup_{\mathcal{K}\subseteq\mathcal{N} }\left\vert \sum_{i\in\mathcal{K}} d_i-\sum_{i\in \mathcal{K}} \mathbb{P}(i_k=i\mid i_0=j) \right \vert\leq M_1 \exp(-k/M_2)
    \end{align}
    for all $j\in\mathcal{L}$ and $ k\geq 0$.}
\end{assumption}
\begin{assumption}\label{assump:contraction}
   \textit{The Operator $F(\cdot)$ is a $\gamma$-contraction with respect to $\Vert\cdot \Vert_{\infty}$. Further, there exists some constant $C>0$ such that $\Vert F(x) \Vert_\infty \leq \gamma \Vert x \Vert_\infty+C$ for all $x \in \mathbb{R}^{n}$.}
\end{assumption}

\begin{assumption}\label{assump:martingale}
    \textit{The random process $w(k)$ is $\mathcal{F}_{k+1}$ measurable and satisfies $\mathbb{E}[w(k) \mid \mathcal{F}_k]=0$. Further, $\vert w(k) \vert\leq \Bar{w}$ almost surely for some constant $\Bar{w}$.}
\end{assumption}

The following theorem shows that the stochastic approximation presented in Eq. (\ref{eq:sa_yiheng}) converges to the unique $x^*$ that solves
\begin{align*}
    \Pi F(\Phi x^*)=x^*.
\end{align*}
where $\Pi:=(\Phi^\mathsf{T}D\Phi)^{-1}\Phi^\mathsf{T}D$. Here, $D$ is the diagonal matrix with diagonal components being the stationary distribution of the Markov chain $\{i_k\}$, which we denote by $d$.
\begin{theorem}
    Suppose that Assumptions \ref{assump:egodicity}, \ref{assump:contraction}, \ref{assump:martingale} hold. Further, assume that there exists some constant $\Bar{x}\geq \Vert x^* \Vert_\infty$ such that $\Vert x(k) \Vert_\infty \leq \Bar{x}$ for all $k$ almost surely. Let the stepsize be $\alpha_k=\frac{H}{k+k_0}$ with $k_0=\max(4H, 2M_2\log K)$ and $H\geq \frac{2}{\sigma' (1-\gamma)}$. Define the constants $C_1:=2\Bar{x}+C+\Bar{w}, C_2:=4\Bar{x}+2C+\Bar{w}, C_3:=2M_1 (2\Bar{x}+C)(1+2M_2+4H)$. Then, with probability at least $1-\delta'$, we have
    \begin{align*}
        \Vert x(K)-x^* \Vert_\infty \leq \frac{C'_a}{\sqrt{K+k_0}}+\frac{C_b}{K+k_0}=\Tilde{O}\left(\frac{1}{\sqrt{K}}\right),
    \end{align*}
    where 
    $$C'_a=\frac{4HC_2}{1-\gamma}\sqrt{K_2\log K \left( \log \left( \frac{4lM_2K}{\delta'} \right)+\log\log K \right)},$$
    $$C_b=4\max\left\{ \frac{48M_2C_1H\log K +\sigma' C_3}{(1-\gamma)\sigma'}, \frac{2\Bar{x}(2M_2 \log K +k_0)}{1-\gamma}\right\}.$$
\end{theorem}

\subsection{Extensions}\label{subsec:beyond_IL}
Our proof follows from the following blueprint: (1) constructing a separable MDP to approximate the original MDP, (2) analyzing the algorithm as if it were implemented on the separable MDP, and (3) bounding the dap due to model difference between the separable MDP and the original one. This framework can be used to study other learning scenarios beyond IL considered in this paper. We briefly elaborate on some of them in the following.

\paragraph{Partially Observable Markov Decision Processes (POMDPs)} In sequential decision-making problems, sometimes the agent (or agents) is not able to observe the complete state and must work with a partial observation of the state \cite{littman2009tutorial}. More specifically, instead of observing a trajectory of state-actions $\{(S_k,A_k)\}$, the agent actually observes $\{(O_k,A_k)\}$, where $O_k \in \mathcal{O}$ denotes the partial observation at time $k$. In this case, unlike $\{(S_k,A_k)\}$, the trajectory $\{(O_k,A_k)\}$ does \textit{not} form a Markov chain, which imposes major technical difficulties in analyzing POMDPs. In light of our analysis framework, if we can approximate the non-Markovian trajectory $\{(O_k,A_k)\}$ by a Markov chain that is defined on $\mathcal{O} \times \mathcal{A}$, then, it would enable a tractable analysis of POMDPs. In this case, the approximation error (which is analogous to our dependence level) will appear in the analysis. Studying POMDPs is another future direction of this work.

\paragraph{General Non-Markovian Stochastic Iterative Algorithms} RL algorithms, in general, belong to the broad framework of stochastic iterative algorithms, also known as stochastic approximation algorithms \cite{robbins1951stochastic}, which can be represented in the following form
\begin{align*}
    x_{k+1}=x_k+\alpha_k G(x_k,Y_k),
\end{align*}
where $Y_k$ is the noise and $G(\cdot,\cdot)$ is some properly defined operator, depending on the problem of interest. As we see in Section \ref{subsec:proof_critic}, the update equation for the joint Q-function can be written as a stochastic approximation in the form above. Moreover, the popular stochastic gradient descent algorithm in large-scale continuous optimization is also a special case of stochastic approximation, where $G(\cdot,\cdot)$ is a noisy version of the negative gradient operator. While stochastic approximation has been extensively studied in the literature, most if not all existing results assume the noise sequence $\{Y_k\}$ is either an i.i.d. sequence or forms a Markov chain. In light of our proof framework, we are able to analyze stochastic approximation algorithms driven by non-Markovian samples, as long as the trajectory $\{Y_k\}$ can be approximated by a Markov chain. Rigorously developing asymptotic and finite-sample convergence results for non-Markovian stochastic approximation algorithms is also an interesting future direction.

\section{Examples and Experiments}\label{append:examples_and_experiemnt_setting}
This section provides the details of the illustrative example presented in Section \ref{ex:artificial} and the numerical simulations.

\subsection{Details of the Synthetic MDP}\label{append:detailed_examples1}
\paragraph{Original MDP.} The state of the three agents is a $3$-dimension vector denoted as $s=(s^1,s^2,s^3)$, where $s^i\in \{0,1\}$ for $i\in \{1,2,3\}$. The action is denoted as $a=(a^1,a^2,a^3)$, where $a^i\in \{0,1\}$ for $i\in \{1,2,3\}$. For each agent, one unit of reward can be obtained when the value of the state variable stays unchanged and $0$ is obtained otherwise. Agents $1$ and $2$ are strongly coupled and always have the same state value. Their state values remain unchanged when they take the same actions. Otherwise, their state values will transit from $0$ to $1$ or from $1$ to $0$ simultaneously. That means the following holds when $s_k^1=s_k^2$,
\begin{align*}
    \mathbb{P}(s^1_{k+1}=s^1_k,s^2_{k+1}=s^2_k\mid a^1_k=a^2_k)=1,\quad
    \mathbb{P}(s_{k+1}^1\ne s_k^1,s_{k+1}^2\ne s_k^2\mid a_k^i\ne a_k^2)=1.
\end{align*}
For agent 3, any action leads to the state value $0$ or $1$ with probability $0.5$ when (1) its state value is $1$ or (2) its state value is $0$ and it applies the same action to agent $2$. However, when its state value is 0 and it applies a different action compared with agent 2, the next state will be $1$ with probability $1$. The transition model for agent 3 is summarized as
\begin{align*}
    \mathbb{P}(s^3_{k+1}=0\mid s_k^3=0,a^2_k\ne a^3_k)=\;&0,\quad& \mathbb{P}(s^3_{k+1}=1\mid s_k^3=0,a^2_k\ne a^3_k)=\;&1,\\
    \mathbb{P}(s^3_{k+1}=0\mid s_k^3=0,a^2_k = a^3_k)=\;&0.5,\quad &\mathbb{P}(s^3_{k+1}=1 \mid s_k^3=0,a^2_k= a^3_k)=\;&0.5,\\
    \mathbb{P}(s^3_{k+1}=0\mid s_k^3=1)=\;&0.5,\quad &\mathbb{P}(s^3_{k+1}=1 \mid s_k^3=1)=\;&0.5.
\end{align*}
The different dependencies among agents induce different dependence levels with different grouping options.

\paragraph{The Dependence level.} Suppose that we group agents $1$ and $2$ together, that is, they can be seen as one agent controlling 2 action variables. Then, it is easy to derive $\mathcal{E}=0.5$ by solving the optimization problem in Eq. \eqref{trans_approx}. Suppose that agents $2$ and $3$ are grouped together. Then, we have $\mathcal{E}=0.75$. Similarly, we have $\mathcal{E}=0.875$ when agents $1$ and $3$ are grouped together. The detailed separable MDPs for different grouping options are shown in Figure \ref{fig:example1_new_model}.
\begin{figure}[H]
    \centering
    \includegraphics[width=6in]{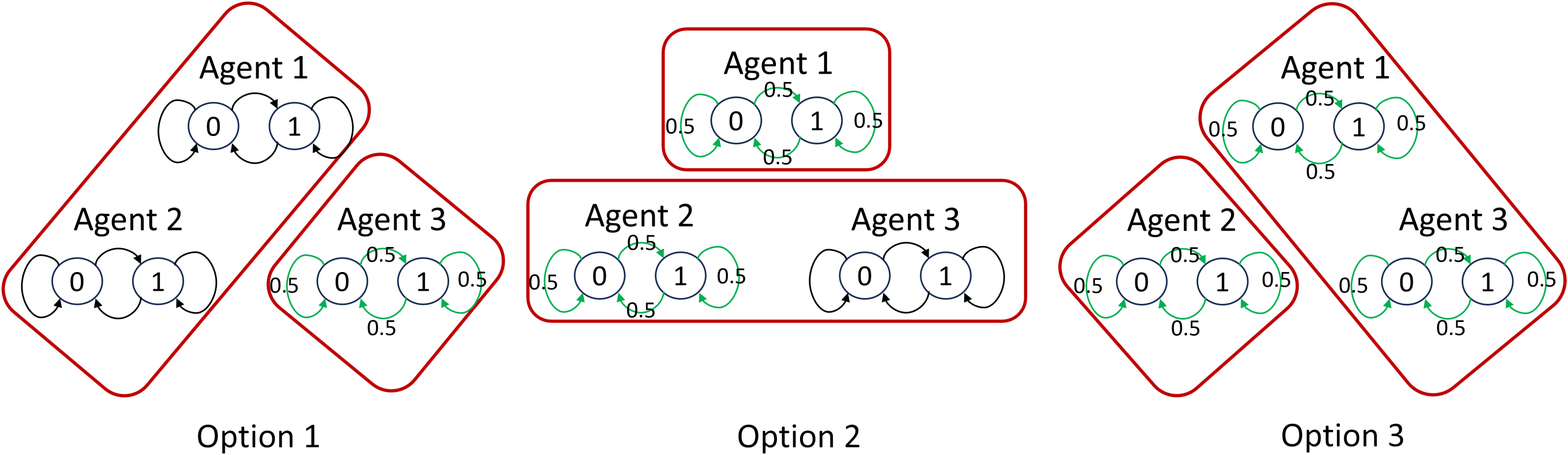}
    \caption{The separable MDPs for different grouping options}
    \label{fig:example1_new_model}
\end{figure}
In Figure \ref{fig:example1_new_model}, black arrows mean the same transition probabilities as the original MDP, and green arrows mean random transition probabilities with 0.5 to any state under any state and action. 

\subsection{The EV Charging Problem} \label{append:detailed_examples2}
Consider the problem where there are multiple charging stations in a distributed network and the total charging power must remain within a time-varying capacity to guarantee safety. A natural approach to address this problem is to formulate charging coordination as a hierarchical resource allocation problem, where the resource is the charging capacity. Each charging station receives an allocated charging capacity which denotes the maximum power that can be used to charge an EV. Consequently, safety is guaranteed because the total charging power consumed by all stations is upper bounded by the total charging capacity.

Consider allocating the charging capacity through a tree structure shown in Figure \ref{fig:binary_tree}, where decisions are made layer by layer. Agents in the same layer perform IL and decide the charging capacity assigned to the children. Specifically, as illustrated in Figure \ref{fig:binary_tree}, the maximum charging power is first allocated to the top agent and it needs to decide the capacities allocated to its left and right children. Each child receives a maximum charging capacity, which means that the total charging power induced by the charging stations on the child's side must be within the limit. Then, similar allocations are performed on all lower layers. Absolute safety is guaranteed since the charging capacity is allocated layer by layer to make sure the total charging power is within the upper limit. Finally, all charging stations on leaves charge the connected EVs no more than the maximum charging capacity they receive.
\begin{figure}[htbp]
    \centering
    \includegraphics[width=4in]{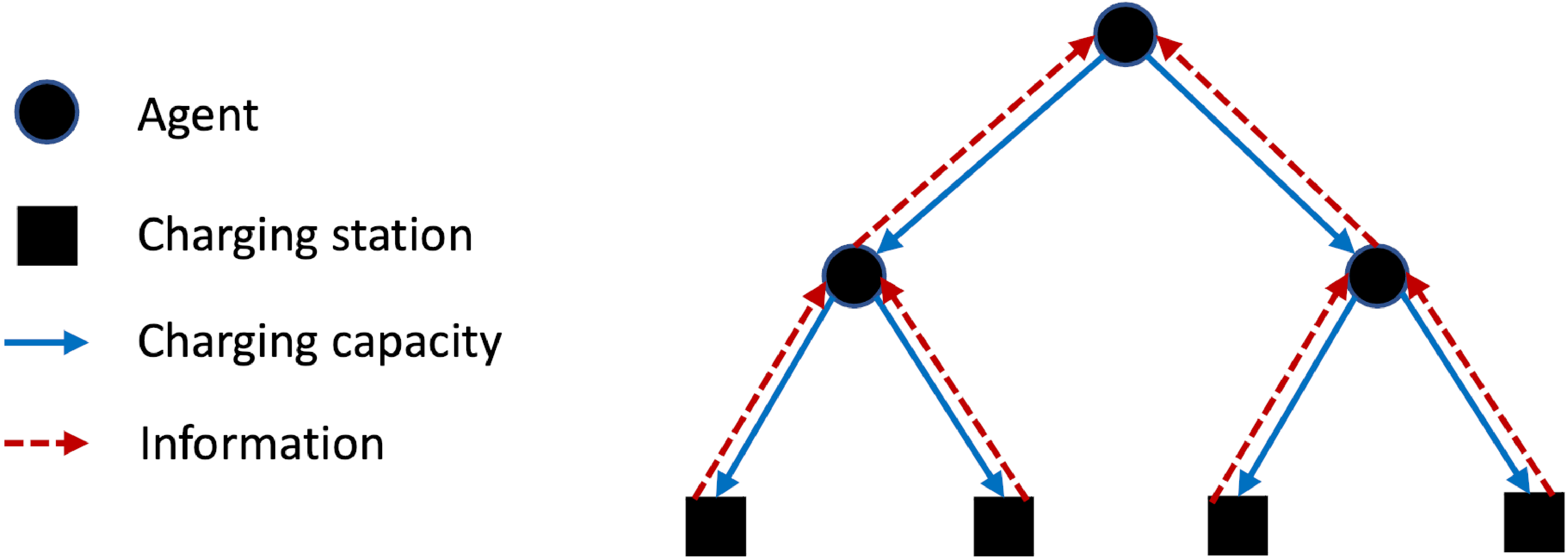}
    \caption{Illustration of a binary tree to allocate the charging capacity.}
    \label{fig:binary_tree}
\end{figure}

To make decisions, each agent observes the charging capacity assigned to it and the information of its children. The information from children is contracted from high dimensions to low dimensions for scalability before being transferred to the parents. 
In this formulation, the problem of EV charging is captured by an MARL with various dependence levels. To see this, consider agents in the same layer. Their parents' and children's policies can be seen as stationary by our design of the multi-scale learning rates (see Section \ref{append:detailed_experiment_setting} for a detailed discussion). Information contraction controls the dependence level $\mathcal{E}$ among the agents in the same layer, which leads to different performance gaps in our analysis. To further elaborate, consider the following two extreme cases. In one extreme case, the contraction keeps the information unchanged, in which case the agents within the same layer are strongly correlated because their parent makes decisions based on their full information. In another extreme case, no information is communicated to the parents. As a result, all agents in the same layer are completely independent. 

In our setting, all agents cooperate to maximize the long-term discounted total reward of all charging locations. Each agent can only communicate locally with its parent or children.

Each charging station $j$ has its local information $I^j(t)=(p^j(t),d^j(t),f^j(t),l^j(t))$ of the connected EV at time $t$, where $p^i(t)$ is the rated charging power, $d^j(t)$ is the proportion of satisfied demand, $f^j(t)$ and $l^j(t)$ are the additional time needed to fully charge the EV and the remaining time before the EV leaves. All variables are 0 if no EV is connected or the connected EV is fully charged. The charging station $j$ can induce a reward of
$r^j(t)=\Delta d^j(t)=d^j(t+1)-d^j(t)$
after charging the connected EV with power $P_c^j(t)=\min\{P^j(t),p^j(t)\}$, where $P^j(t)$ denotes the maximum charging capacity it receives. The reward is $r^j(t)=0$ if the connected EV is not charged or no EV is connected.

We define the states, actions, and rewards for the agents as follows:
\begin{itemize}
    \item \textbf{States.} For each agent $i$, the state includes the information from its children and the charging capacity allocated to it, which can be written as $s^i(t)=(I^{C_l^i}(t),I^{C_r^i}(t),P^i(t))$, where $C_l^i, C_r^i$ denote the left and right children of agent $i$. The information of agent $i$ is constructed recursively from the bottom level to the top as a function of its children's information. Specifically, we define $I^i(t)=\mathbf{W}^i_1(t) (I^{C_l^i}(t))^T+\mathbf{W}^i_2(t) (I^{C_r^i}(t))^T$, where $\mathbf{W}^i_1(t),\mathbf{W}^i_2(t)$ are both matrices that can be designed to control the information communicated to the parents. We will discuss more about how to control the dependence level with different choices of $\mathbf{W}^i_1(t),\mathbf{W}^i_2(t)$ below.
    \item \textbf{Actions.} The action of agent $i$ at time $t$ is the proportion of charging capacity allocated to the left child $a^i(t)\in[0,1]$. Consequently, the proportion of charging capacity allocated to the right child is $1-a^i(t)$.
    \item \textbf{Rewards.} The reward of agent $i$ is defined recursively as the sum of the rewards of its children from the bottom level to the top level. That is, $r^i(t)=r^{C_l^i}(t)+r^{C_r^i}(t)$. Thus, the reward of agent $i$ is equal to the total reward of all charging stations that have $i$ as an ancestor.
\end{itemize}

\subsection{Numerical Simulations}\label{append:detailed_experiment_setting}
In this section, we present the details of our numerical experiments.
\subsubsection{Experimental Setting for the Synthetic MDP}\
In this case, we consider the synthetic MDP in Example \ref{ex:artificial}. Three agents are grouped into two with different options according to Figure \ref{fig:example1}. For both IQL and INAC, we will repeat the training and testing 100 times, respectively. Each run includes 3000 steps for training and 1000 steps for testing. The discount factor is $\gamma=0.99$. Specifically, the inner loop of INAC has $K=100$ steps. We set $\alpha=0.05, k_0=4\alpha$ and $\alpha_k=\alpha/(k+k_0)$ for the critic. The stepsize for the actor is set as $\eta_0=0.2$ and $\eta_t=\sum_{i<t} \eta_i/\gamma^{2t-1}$. To encourage exploration during the training process, each agent takes action uniformly at random with probability $\epsilon_k=(1-k/K)/10$. For IQL, we also set $\alpha=0.05, k_0=4\alpha$ and $\alpha_k=\alpha/(k+k_0)$. The policy $\pi_b$ is set as the uniform distribution over the action space. We calculate the normalized reward relative to the optimal expected reward of each episode consisting of 100 steps. The optimal reward can be easily computed with the model described in Appendix \ref{append:detailed_examples1}.

\subsubsection{Experimental Setting for EV Charging}
\begin{figure}[htbp]
    \centering
    \includegraphics[width=4.5in]{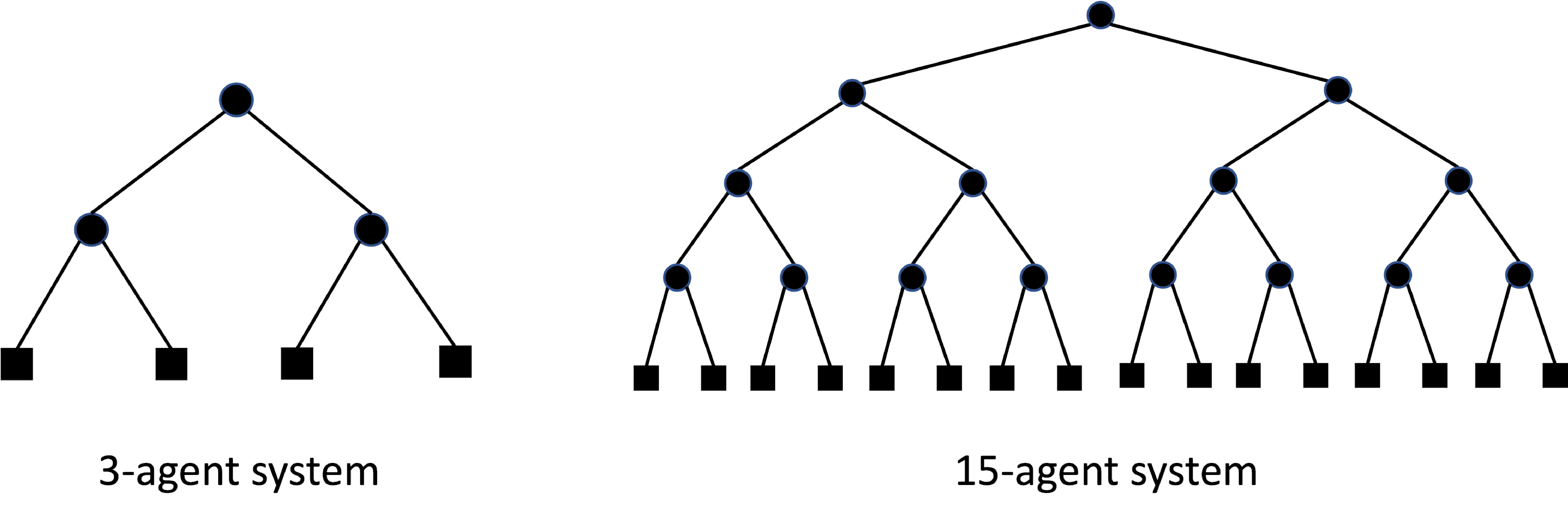}
    \caption{The structures of the 3-agent system (left) and the 15-agent system (right). The square nodes represent the charging stations and the circle nodes represent the agents. Note that our results do not require the interaction structure to be full binary trees.}
    \label{fig:3_15_agent_systems}
\end{figure}
Here we consider the 3-agent and 15-agent systems shown in Figure \ref{fig:3_15_agent_systems}. Recall that $\mathbf{W}^i_1(t),\mathbf{W}^i_2(t)$ can be designed to derive different information contractions with different dependence levels among agents within the same layer. We design three types of information contractions to perform IL, including full information, average information, and no information. Full information means that each agent communicates the full information from its children to its parent, where the dimensions of agents' states increase exponentially from the bottom to the top. Average information means that the information communicated to the parent is taken as a weighted average of the information from its children. No information means that each agent communicates nothing to its parent. If $\mathbf{W}^i_1(t)=0,\mathbf{W}^i_2(t)=0$, the children do not send any information to their parents. The agent can also send average information by mixing the information of its children with $\mathbf{W}^i_1(t)=\text{diag}(1,e_l^i(t),e_l^i(t),e_l^i(t))$ and $\mathbf{W}^i_2(t)=\text{diag}(1,e_r^i(t),e_r^i(t),e_r^i(t))$, where
\begin{align*}
    e_l^i(t)=\frac{p^{C_l^i}(t)}{p^{C_l^i}(t)+p^{C_r^i}(t)},\text{ and }e_r^i(t)=\frac{p^{C_r^i}(t)}{p^{C_l^i}(t)+p^{C_r^i}(t)}.
\end{align*}
We set $e_l^i(t)=e_r^i(t)=0$ if $p^{C_l^i}(t)=p^{C_r^i}(t)=0$. In other words, the contracted $d^i(t), f^i(t),$ and $l^i(t)$ are the weighted average of information from its children according to the rated charging power. The agent sends the full information to its children when $\mathbf{W}^i_1(t)=(\mathbf{I}_{n^i},\mathbf{0})^T\in \mathbb{R}^{2n^i \times n^i},\mathbf{W}^i_2(t)=(\mathbf{0},\mathbf{I}_{n^i})^T\in\mathbb{R}^{2n^i\times n^i}$, where $\mathbf{I}_{n^i}$ is the identity matrix, and $n^i$ is the number of dimensions of the information from agent $i$'s children.

EV arrivals are simulated with fixed arrival rates sampled from $(0,1)$ for each charging station. Each EV is set with random charging demand, maximum charging power, and remaining time before leaving. We set the remaining time before leaving longer than the time needed to fully charge the EV. The time interval between two decisions is set to 1 hour.

Two other policies are considered for comparison to our algorithms. The first one is the offline optimal policy, which is non-causal and is computed via linear programming using all information collected throughout the time window. The second is a heuristic baseline policy similar to the business-as-usual policy that charges the EV immediately upon arrival \cite{sadeghianpourhamami2019definition}, but due to the tree structure and the safety concern, the baseline policy here selects actions based on the proportion of rated power of the children, i.e., $a^i(t)=e_l^i(t)$.

\paragraph{Parameters of IQL.} Considering the continuity of the state space, we use two neural networks to learn the $Q$-functions. The current network is used to provide decisions and interact with the environment, and the target network is used to learn the optimal $Q$-fucntion. The parameters of the current network are set as the copy of the target network every 1000 steps. The target network is updated every 5 steps. The action space $[0,1]$ is discretized as $\{0, 0.1, \cdots, 1\}$. We also apply multiscale learning rates in different layers to make sure the agents in other layers are relatively stationary to the agents in each layer. Then, only the agents within the same layer are relatively non-stationary to each other. The learning rates are set to increase with depth by a multiplier of 10 for Q-networks and policy networks. To encourage exploration, the actions are given randomly with a linearly decreasing probability during the training process. The detailed parameters are listed in Table~\ref{table:IQL_parameter}.

\begin{table}[]
\begin{center}
\caption{Parameters of IQL for 3-agent and 15-agent systems}
\label{table:IQL_parameter}
\begin{tabular}{c|cc} 
\hline
                          & 3-agent system             & 15-agent system            \\ \hline
Network                   & \multicolumn{2}{c}{3 hidden layers, 64 neurons for each layer} \\ \hline
Learning rate (Top agent) & $10^{-4}$ & $10^{-5}$ \\ \hline
Max. exploration prob.    & 1                          & 0.1                        \\ \hline
Min. exploration prob.    & 0.03                       & 0.03                       \\ \hline
Batch size                & 32                         & 64                         \\ \hline
Buffer size               & 5000                       & 6000 
\\ \hline
\end{tabular}
\end{center}
\end{table}

\paragraph{Parameters of INAC.} Due to the continuity of the state and action space, we also use neural networks to act as the actor and critic. Similarly to IQL, we apply multiscale learning rates in different layers for INAC. The learning rates are also set to increase with depth by a multiplier of 10 for Q-networks and policy networks. To encourage exploration, we add Guassian noise to the action with mean 0 and decreasing variance with steps. The detailed parameters are shown in Table~\ref{table:INAC_parameter}.

\begin{table}[]
\begin{center}
\caption{Parameters of INAC for 3-agnet and 15-agent systems}
\label{table:INAC_parameter}
\begin{tabular}{c|cc}
\hline
                                     & 3-agent system             & 15-agent system            \\ \hline
Network                              & \multicolumn{2}{c}{3 hidden layers, 64 neurons for each layer} \\ \hline
Learning rate for critic (Top agent) & $10^{-4}$ & $10^{-6}$ \\ \hline
Learning rate for actor (Top agent)  & $5\times 10^{-5}$                          & $5\times 10^{-7}$                          \\ \hline
Max. variance                        & 10                         & 10                         \\ \hline
Min. variance                        & 0.4                        & 0.4                        \\ \hline
Batch size                           & 32                         & 32                         \\ \hline
Buffer size                          & 4000                       & 4000                       \\ \hline
\end{tabular}
\end{center}
\end{table}

In our simulations, both IQL and INAC are trained and tested for 20 times. Each run includes 40000 steps, and each episode includes 10 days (240 steps). The first 125 episodes are for training, and the rest of the episodes are for testing. The normalized reward is calculated relative to the optimal average reward of the offline optimal policy.

\subsubsection{Results for EV Charging}\label{subsec:ev_charging_results}
We repeated the experiments 20 times in both the 3-agent system and the 15-agent system, respectively, and the results are shown in Figure \ref{fig:IQL_combined_result} and Figure \ref{fig:INAC_combined_result}. The rewards are also normalized by the average optimal reward calculated by the offline optimal policy. The results in the 3-agent system show that IQL and INAC exhibit similar performances. Their performances with different information contractions are very close and much better than the baseline policy, and the performances with full information are slightly better. This phenomenon indicates that we can aggressively contract the information without compromising the optimality. In addition, the performance of IQL is significantly better than that of INAC.

\begin{figure}[htbp]
    \centering
    \includegraphics[width=5.5in]{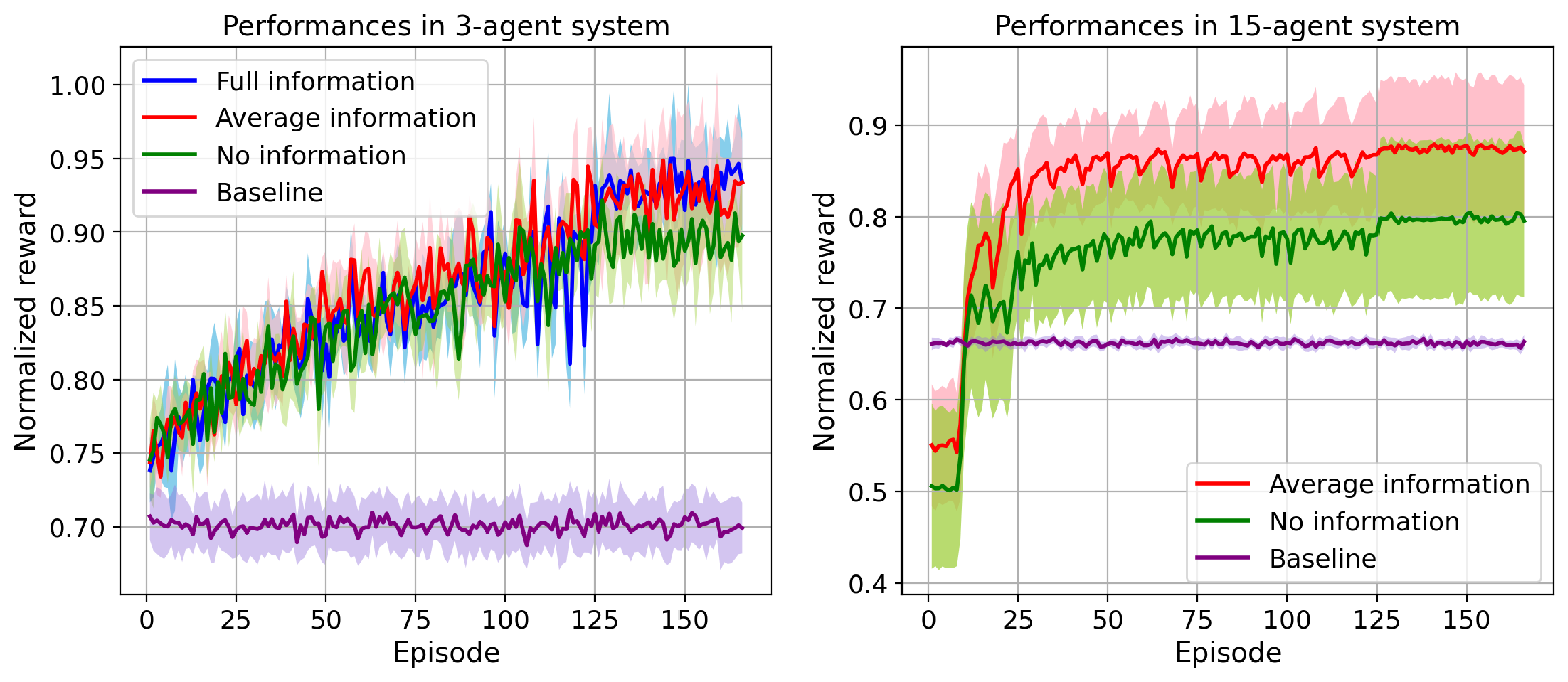}
    \caption{Performance of IQL with different contractions in the 3-agent (left) and 15-agent (right) systems.}
    \label{fig:IQL_combined_result}
\end{figure}

\begin{figure}[h]
    \centering
    \includegraphics[width=5.5in]{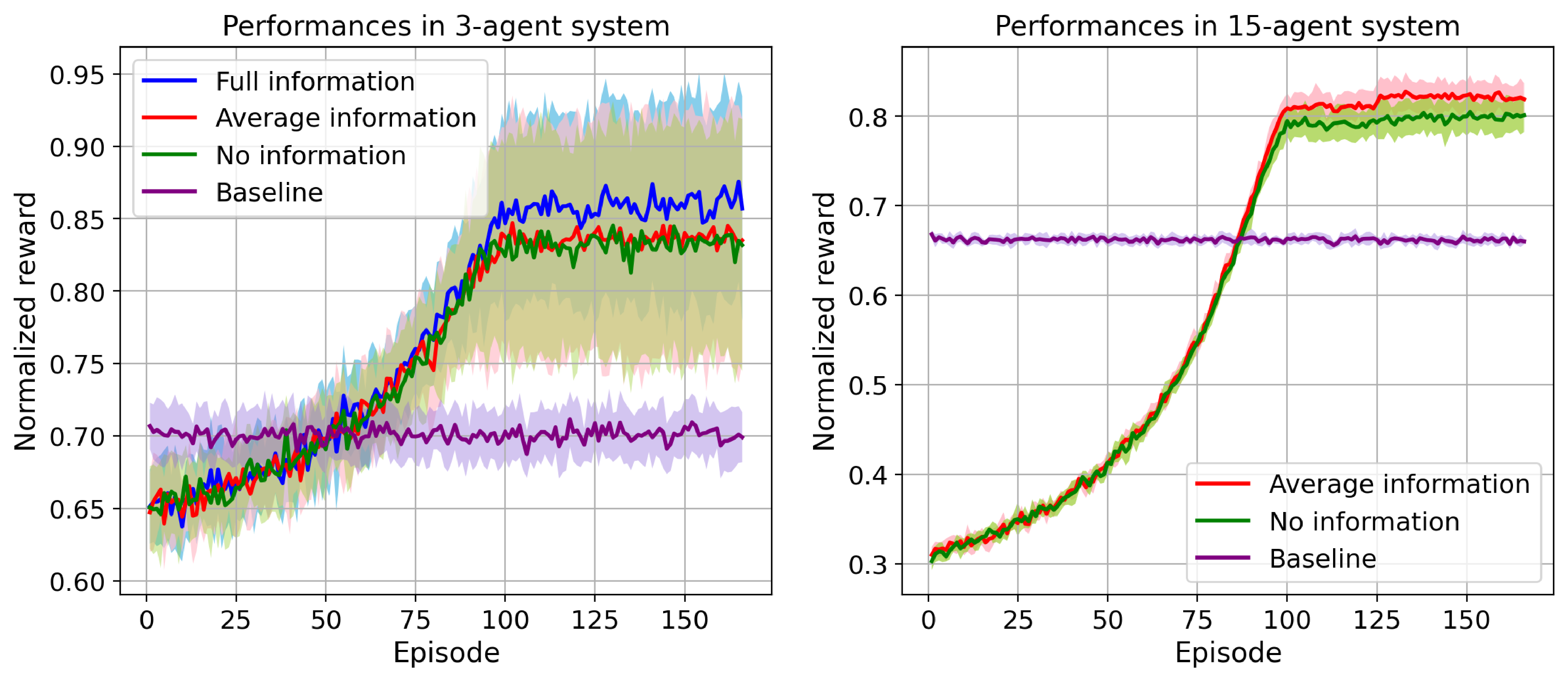}
    \caption{Performance of INAC with different contractions in the 3-agent (left) and 15-agent (right) systems.}
    \label{fig:INAC_combined_result}
\end{figure}

In the 15-agent system, IQL or INAC with no information contraction is computationally intractable due to the curse of dimensionality. The results show that the performance of INAC with average information and no information are both around 80\% of the optimal rewards, where the optimal rewards are calculated by an offline optimal algorithm (see Appendix \ref{append:detailed_experiment_setting}), and the former is slightly better. The performance of IQL with average information can significantly outperform others.

\end{document}